\documentclass{article}

\usepackage{microtype}
\usepackage{graphicx}
\usepackage{epstopdf}
\usepackage{booktabs} 
\usepackage{psfrag,epsfig}
\usepackage{xspace}

\usepackage{hyperref}

\usepackage[accepted]{icml2019}

\icmltitlerunning{GNI Formulation for Games}

\usepackage{amsmath}
\usepackage{amsfonts}
\usepackage{amsthm}
\usepackage{subfigure}

\newtheorem{lemma}{Lemma}
\newtheorem{assumption}{Assumption}
\newtheorem{theorem}{Theorem}
\newtheorem{corollary}{Corollary}

\newcommand{\R}{\mathbb{R}}
\newcommand{\half}{\frac{1}{2}}
\newcommand{\bfx}{\boldsymbol{x}}
\newcommand{\bfy}{\boldsymbol{y}}
\newcommand{\bfzero}{\boldsymbol{0}}
\newcommand{\prob}{P}
\DeclareMathOperator*{\diag}{diag}
\newcommand{\expect}{\mathbb{E}}
\newcommand{\norm}[1]{\left\|{#1}\right\|}
\newcommand{\dist}[1]{d_{#1}}

\newcommand{\bfQ}{\boldsymbol{Q}}
\newcommand{\bfr}{\boldsymbol{r}}

\begin{document}

\twocolumn[
\icmltitle{Game Theoretic Optimization via Gradient-based Nikaido-Isoda Function }



\icmlsetsymbol{equal}{*}

\begin{icmlauthorlist}
\icmlauthor{Arvind U. Raghunathan}{merl}
\icmlauthor{Anoop Cherian}{merl}
\icmlauthor{Devesh K. Jha}{merl}
\end{icmlauthorlist}

\icmlaffiliation{merl}{All authors are with Mitsubishi Electric Research Labs (MERL), Cambridge, MA}

\icmlcorrespondingauthor{Arvind U. Raghunathan}{raghunathan@merl.com}
\icmlcorrespondingauthor{Anoop Cherian}{cherian@merl.com}
\icmlcorrespondingauthor{Devesh K. Jha}{jha@merl.com}

\icmlkeywords{Game Theory, Nikaido Isoda, Non-cooperative games, non-convex optimization, Nash Equilibrium}

\vskip 0.3in
]



\printAffiliationsAndNotice{}  

\begin{abstract} 
Computing Nash equilibrium (NE) of multi-player games has witnessed renewed interest due to recent advances in generative adversarial networks. However, computing equilibrium efficiently is challenging. To this end, we introduce the \emph{Gradient-based Nikaido-Isoda } (GNI) function which serves: (i) as a merit function, vanishing only at the first-order stationary points of each player's optimization problem, and (ii) provides error bounds to a stationary Nash point. Gradient descent is shown to converge sublinearly to a first-order stationary point of the GNI function. For the particular case of bilinear min-max games and multi-player quadratic games the GNI function is convex.  Hence, the application of gradient descent in this case yields linear convergence to an NE (when one exists). In our numerical experiments we observe that the GNI formulation always converges to the first-order stationary point of each player's optimization problem.
\end{abstract} 

\section{Introduction}\label{sec:introduction}
In this work, we consider the general $N$-player game:
\begin{equation}
\begin{aligned}
	\text{Find } \bfx^* = (x_1^\star,\ldots,x_N^\star) \text{ s.t. }  \\
	x_i^\star = \arg\min\limits_{\bfx \in \R^n : \bfx_{-i} = \bfx_{-i}^*} f_i(\bfx) 
\end{aligned}\label{gameForm}	
\end{equation}
where $x_i \in \R^{n_i}$, $n = \sum_{i=1}^N n_i$, $f_i : \R^{n} 
\rightarrow \R$,  
$\bfx = (x_1,\ldots,x_N) \in \R^n$ denotes the collection of all $x_j$'s, 
while $\bfx_{-i}$ 
denotes the collection of all $x_j$'s except for index $i$, \emph{i.e.} 
$\bfx_{-i} = (x_1,\ldots,x_{i-1},x_{i+1},\ldots,x_N) \in \R^{(n-n_i)}$.  
Observe that the choice of $\bfx_{-i}$ are 
specified when performing the minimization in~\eqref{gameForm} for player $i$.  

A point $\bfx^\star $ satisfying~\eqref{gameForm} is called a 
\emph{Nash Equilibrium} (NE).  
We denote by ${\cal S}^{NE}$ the set of all NE points, \emph{i.e.,} 
${\cal S}^{NE} = \{ \bfx^\star \,|\, \eqref{gameForm} \text{ holds} \}$.  
In the absence of convexity for the 
functions $f_i$ we may not be able to obtain a minimizer 
in~\eqref{gameForm} and have to settle for a first-order stationary 
point. 
Accordingly, define ${\cal S}^{SNP}$ to be the set of all \emph{Stationary  
Nash Points}, \emph{i.e.,} 
${\cal S}^{SNP} = \{ \bfx^\star \,|\, \nabla_i f_i(\bfx^\star) = 0;\forall\; i = 1,\ldots,N \}$ 
where $\nabla_i f$ denotes the derivative of function $f$ \emph{w.r.t.} $x_i$.

There has been renewed interest in Nash equilibrium computation 
for games owing to the success of 
Generative Adversarial Networks (GANs) \cite{GANfirst}.  
GANs have been successful in learning probability 
distributions and have found application in tasks including image-to-image 
translation~\citep{GANimage2image},  
domain adaptation~\citep{GANdomain}, probabilistic inference~\citep{GANprobinf1,GANprobinf2} among others.  
Despite their popularity, GANs are known to be difficult to train. 
In order to stabilize training recent 
approaches have resorted to carefully designed models, either by 
adapting an 
architecture~\citep{GANarchmod} or by selecting an easy-to-optimize objective function~\citep{GANobjmod1,GANobjmod2,GANobjmod3}.  

The Nikaido-Isoda (NI) function~\cite{NikaidoIsoda} (formally introduced in \S\ref{sec:GNI}) is popular 
in equilibrium  computation~\cite{NI_UryasevRubinstein,Krawczyk2004,FaccKanzow07,HeusingerKanzow1,HeusingerKanzow2} and often used as a merit function for NE.
The evaluation of the NI function requires optimizing each player's problem 
globally which can be intractable for non-convex objectives.   

In this paper, we introduce \textit{Gradient-based Nikaido-Isoda} (GNI) function which allows us to computationally simplify the original NI formulation. Instead of computing a globally optimal solution, every player 
can locally improve their objectives using the steepest descent 
direction. 
The proposed GNI function simplifies the original NI formulation by relaxing the requirement on optimizing individual player's objective globally. We prove that GNI is a valid merit function for multi-player games and vanishes only at the first-order stationary points of each player's optimization 
problem (\S\ref{sec:GNI}). 
The GNI function is shown to be locally stable in a 
neighborhood of a stationary Nash point (\S\ref{sec:stability}) and 
convex when the player's objective function is quadratic 
(\S\ref{sec:gniconvex}).  
The gradient descent algorithm applied to the GNI function converges 
to a stationary Nash point (\S\ref{sec:algorithm}). 
In addition, if each of the player's objective is convex in the 
player's variables ($x_i$)  then the algorithm converges to the 
NE point as long as one exists (\S\ref{sec:algorithm}). 
A secant approximation is provided to simplify the computation of 
the gradient of the GNI function and the convergence of the modified 
algorithm is also analyzed (\S\ref{sec:moddescent}).  
Numerical experiments in \S\ref{sec:expts} show that the proposed 
algorithm is effective in converging to stationary Nash points 
of the games.

We believe our proposed GNI formulation could be an effective approach for training  GANs. However, we emphasize that the focus of this paper is to provide a rigorous analysis of the GNI formulation for games and explore its properties in a non-stochastic setting. The adaptation of our proposed formulations to a stochastic setting (which is the typical framework commonly used in GANs) will need additional results, which will be explored in a future paper.


\section{Related Work}\label{sec:related_work}
Nash Equilibrium (NE) computation, a key area in algorithmic game theory, has seen a number of 
developments  since the pioneering work of John von Neumann~\cite{basar1999dynamic}.  
It is well known that the Nash equilibrium problem can be reformulated as a variational inequality problem, VIP for short, see, for example,~\cite{finiteVIbook1}.  The VIP is a generalization of 
the first-order optimality condition in ${\cal S}^{SNP}$ to the case where the decision 
variables of player $i$'s $x_i$ are constrained to be in a convex set.  
\citet{FaccKanzow10} proposed penalty methods for the solution of generalized Nash 
equilibrium problems (Nash equilibrium problems with joint constraints).  
\citet{IusemJofre17} provides a detailed analysis of the extragradient algorithm for 
stochastic pseudomonotone variational inequalities (corresponding to games with 
pseudoconvex costs).
 
Nash Equilibrium computation has found renewed interest due to the emergence of Generative Adversarial Networks (GANs). It has been 
observed that the alternating stochastic gradient descent (SGD) 
is oscillatory when training GANs~\cite{Goodfellow16}.  Several papers 
proposed to modify the GAN formulation in order to stabilize the convergence of the iterates.
These include  non-saturating GAN formulation of~\cite{GANfirst,Fedus18}, 
the DCGAN formulation~\cite{GANarchmod}, the gradient penalty 
formulation for WGANs~\cite{GANobjmod3}.  
The authors in~\cite{GANstablize} proposed a momentum based step on the generator in the 
alternating SGD for convex-concave saddle point problems.  
Daskalakis \emph{et al.}~\cite{DaskalakisGAN1} 
proposed the optimistic mirror descent (OMD) algorithm, and showed convergence for bilinear 
games and divergence of the gradient descent iterates.  In a subsequent work, Daskalakis \emph{ et al.}~\cite{DaskalakisGAN2} analyzed the limit points of gradient descent 
and OMD, and showed that the limit points of OMD is a superset of alternating gradient descent.  
\citet{MertZenLec19} generalized and extended the work of~\citet{DaskalakisGAN1} for bilinear games.  
\citet{DualGAN} dualize the GAN objective to reformulate it as 
a maximization problem and \citet{GANprobinf2} add the norm of the gradient in the objective.  
The norm of the gradient is shown to locally 
stabilize the gradient descent iterations in~\citet{NagarajanKolter}.  
\citet{VIGAN} formulate the GAN equilibrium as a VIP and propose an extrapolation technique 
to prevent oscillations.  The authors show convergence of stochastic algorithm under the 
assumption of monotonicity of VIP, which is stronger than the convex-concave assumption in  
min-max games.  
Finally, the convergence of stochastic gradient descent in non-convex games has also 
been studied in \citet{BervBravFaure18,MertZhou19}.

In contrast to existing approaches, the GNI approach does not assume 
monotonicity in the game formulations.  The GNI approach is also closely 
related to the idea of minimizing  residuals~\cite{finiteVIbook1,finiteVIbook2}.  


\section{Gradient-based Nikaido-Isoda Function}\label{sec:GNI}
The Nikaido-Isoda (NI) function introduced in~\cite{NikaidoIsoda} is defined as
\begin{equation*}
	\begin{aligned}
		\psi(\bfx) = \sum\limits_{i=1}^N  \underbrace{\left( f_i(\bfx) - \inf_{\hat{\bfx} \in \R^n 
				: \hat{\bfx}_{-i} = \bfx_{-i} } f_i(\hat{\bfx})\right)}_{=: \psi_i(\bfx)} . 
	\end{aligned}\label{defPsi}
\end{equation*}
From the definition of NI function $\psi(\bfx)$, it is easy to show that $\psi(\bfx) \geq 0$ for all 
$\bfx \in \R^n$.  
Further, $\psi(\bfx) = 0$ is the global minimum which is only achieved if the NE point 
$\bfx^\star = (x_1^*,\ldots,x_N^\star)$ occurs at points where 
$x_i^*$ are global minimizers of 
the respective optimization problems in~\eqref{gameForm}.  A number of 
papers~\cite{NI_UryasevRubinstein,HeusingerKanzow1,HeusingerKanzow2} have 
proposed algorithms that minimize $\psi(\bfx)$ to compute NE points.  However, the 
infimum needed to compute $\psi_i(\bfx)$ can be prohibitive for all but a handful of functions.  
For bilinear min-max games (i.e., $f_1(\bfx) = x_1^Tx_2 = -f_2(\bfx)$), the 
infimum is unbounded below and the approach of minimizing NI fails.  To rectify this recent 
papers have proposed regularized variants~\cite{HeusingerKanzow2}.  
However, the cost of globally minimizing the nonlinear function can still be prohibitive.


To rectify the shortcoming of the NI function, we introduce the Gradient-based Nikaido-Isoda 
(GNI) function
\begin{equation}
	\begin{aligned}
					&	V(\bfx;\eta) =\sum\limits_{i=1}^N  
						\underbrace{f_i(\bfx) - f_i(\bfy(\bfx;i,\eta))}_{=: V_i(\bfx;\eta)} \\
	\end{aligned}\label{defVy}
\end{equation}
\begin{equation*}
	\begin{aligned}
		\text{where } y_j(\bfx;i,\eta) = \left\{ \begin{aligned}
			x_i - \eta \nabla_i f_i(\bfx) & \text{, if } j = i \\
			x_j & \text{, otherwise}.
			\end{aligned} \right.
	\end{aligned}
\end{equation*}  
where $\nabla_i f(\bfx)$ denotes the derivative of function $f$ 
\emph{w.r.t.} $x_i$.

The GNI function is obtained by replacing the infimum in the NI function for player $i$ with a point 
$\bfy(\bfx;i,\eta)$ in the steepest descent direction.  This provides a local measure of decrease 
that can be obtained 
in the objective for player $i$.  The point $\bfy(\bfx;i,\eta)$ is similar in spirit to the 
\emph{Cauchy point} that is used in \emph{trust-region} methods~\citep{NocWrightbook}.  
We will show that any point satisfying 
$V_i(\bfx;\eta) = 0$ also satisfies $\nabla_i f_i(\bfx) = 0$.  To show this, we first provide bounds on 
$V_i(\bfx;\eta)$ in terms of the distance from first-order 
optimality conditions for each of the players. 

We make the following standing assumption.
\begin{assumption}\label{assume:differentiable}
The functions $f_i$ are at least twice continuously differentiable and  
gradients of $f_i$ (i.e., $\nabla f_i$) are Lipschitz continuous with constant $L_f$.
\end{assumption}

\begin{lemma}\label{lemma:errBndV}
$\frac{\eta}{2} \|\nabla_i f_i(\bfx)\|^2 \leq V_i(\bfx;\eta) \leq \frac{3\eta}{2} \|\nabla_i f_i(\bfx)\|^2$ 
for all $\bfx \in \R^{n}$ and $0 < \eta \leq \frac{1}{L_f}$. 
\end{lemma}
\begin{proof}
Using the Taylor's series expansion of $f_i$ around $\bfx$ and substituting for 
$\bfy(\bfx;i,\eta)$, we obtain
\[\begin{aligned}
	&\, f_i(\bfy(\bfx;i,\eta)) 	
=	f_i(\bfx) - \eta \| \nabla_i f_i(\bfx) \|^2 \\
	&\,	+ \eta^2 \int\limits_{0}^1  \nabla_i f_i(\bfx)^T
	\nabla^2_i f_i(\hat{\bfx}(t)) \nabla_i f_i(\bfx) t dt 
\end{aligned}\]
where $\hat{x}_i(t) = x_i - t \eta \nabla_i f_i(\bfx)$ and $\hat{x}_j(t) = x_j$, for $j \neq i$.  
From the Lipschitz continuity of the gradient of $f_i$, we have that 
$-L_f I_{i} \preceq \nabla^2_i f_i(\hat{\bfx}(t)) \preceq L_f I_{i}$, where $I_i$ is the ($n_i \times n_i$) identity matrix. 
Substituting in the above and using $\eta \leq \frac{1}{L_f}$ yields the claim.  
\end{proof}

We now state our main result relating the zeros of $V_i(x;\eta)$ and the first-order critical points 
of the players's optimization problems. 

\begin{theorem}\label{thm:zeroVEqForderCrit}
The global minimizers of $V(\bfx;\eta)$ are all stationary Nash points, i.e., 
$\{\bfx^\star \,|\, V(\bfx^\star;\eta) = 0\} = {\cal S}^{SNP}$ for all 
$0 < \eta \leq \frac{1}{L_f}$.  If the individual functions $f_i$ are convex, then the global minimizers 
of $V(\bfx;\eta)$ are precisely the set ${\cal S}^{NE}$. 
\end{theorem}
\begin{proof}
The nonnnegativity of $V(\bfx;\eta)$ follows from Lemma~\ref{lemma:errBndV}.  Further, 
$V(\bfx;\eta) = 0$ if and only if $\nabla_i f_i(\bfx) = 0$. This proves the claim.  
The second claim follows by noting that ${\cal S}^{NE} = {\cal S}^{SNP}$, if the 
functions $f_i$ are convex.
\end{proof}

Theorem~\ref{thm:zeroVEqForderCrit} shows that the function $V(\bfx;\eta)$ can be employed as a 
merit  function for obtaining a stationary Nash point.  When 
$f_i(\bfx)$ are non-convex, the convergence to first-order point is possibly the best that one can 
hope for.  

We provide the expressions for the gradient and Hessian of $V_i(\bfx;\eta)$ next.  These 
expressions follow from the chain rule of differentiation.
The gradient of $V_i(\bfx;\eta)$ is 
\begin{equation}
\begin{aligned}
	& 	\nabla V_i(\bfx;\eta) =\\
			&\quad\nabla f_i(\bfx) - \left(I - \eta \nabla^2 f_i(\bfx) E_i \right) \nabla f_i(\bfy(\bfx;i,\eta)) \label{defgradV}
		\end{aligned}
\end{equation}
where $E_i = F_i F_i^T$ with $F_i \in \R^{n \times n_i}$ defined as 
$F_i^T = \begin{bmatrix} \bfzero_{n_i \times \sum_{j=1}^{i-1}n_j }  & 
I_{i} & \bfzero_{n_i \times \sum_{j=i+1}^N n_j} \end{bmatrix}$, $I \in \R^{n \times n}$, and $I_{i} \in \R^{n_i \times n_i}$ are identity matrices. The Hessian of 
$V_i(\bfx;\eta)$ is given by
\begin{equation}
	\begin{aligned}
		& \nabla^2 V_i(\bfx;\eta) 
	=		\nabla^2 f_i(\bfx) 
		+ \eta \nabla^3 f_i(\bfx) [E_i \nabla f_i(\bfy_i(\bfx;i,\eta)] \\
		& - (I - \eta \nabla^2 f_i(x) E_i ) \nabla^2  f_i(\bfy(\bfx;i,\eta)) (I - \eta E_i \nabla^2 f_i(\bfx)) 
		\end{aligned}\label{defhessV}
\end{equation}
where $\nabla^3 f_i(\bfx) [d] = \lim_{\alpha \rightarrow 0} \frac{\nabla^2 f_i(\bfx+\alpha d) - 
\nabla^2 f_i(\bfx)}{\alpha}$ is the action of the third derivative along the direction $d$. These expressions will come useful in our analysis to follow.


\subsection{GNI is Locally Stable}\label{sec:stability}
GAN formulations typically result in objective functions $f_i(x)$ that 
are not convex.   
Nagarajan and Kolter~\cite{NagarajanKolter} showed that the gradient descent 
for min-max games is not stable for Wasserstein GANs.  This is due to the  
concave-concave nature of Wasserstein GAN around stationary Nash 
points~\cite{NagarajanKolter}.  
Daskalakis \emph{et al.}~\cite{DaskalakisGAN1} showed that 
the gradient descent 
diverges for simple bilinear min-max games, while the optimistic gradient 
decent algorithm of Rakhlin and Sridharan~\cite{RakhlinSridharan} was shown to be 
convergent.  Daskalakis and Pangeas~\cite{DaskalakisGAN2} further analyzed the 
limit points of gradient descent and optimistic gradient descent using dynamical 
systems theory.  

In this section, we show that at every stationary Nash point, the Hessian of $V(\bfx;\eta)$ is 
positive semidefinite.  This ensures that the points in ${\cal S}^{SNP}$ are all stable limit points for 
the gradient descent algorithm on $V(\bfx;\eta)$.  

\begin{lemma}\label{lemma:hessVpsd}
For $0\leq\eta\leq\frac{1}{L_f}$, $\nabla V^2(\bfx^*;\eta) = \sum_{i=1}^N \nabla^2 V_i(\bfx^\star;\eta)$ is positive semidefinite for all~$\bfx^\star \in {\cal S}^{SNP}$.
\end{lemma}
\begin{proof}
Let $\bfx^\star \in {\cal S}^{SNP}$.  Since $\nabla_i f_i(\bfx^\star) = 0$,
we have that $\bfy(\bfx^\star;i,\eta) = \bfx^\star$ and $\nabla^3 f_i(\bfx^\star)[E_i 
\nabla f_i(\bfy(\bfx;i,\eta))] = 0$.  Substituting in the expression for $\nabla^2 V_i(\bfx;\eta)$ 
in~\eqref{defhessV} and simplifying, we obtain
\begin{align}
	 \nabla^2 V_i(\bfx^\star;\eta)&=2\eta \nabla^2 f_i(\bfx^\star) E_i  \nabla^2  f_i(\bfx^\star) \nonumber \\
	 &\quad - \eta^2 \nabla^2 f_i(\bfx^\star) E_i  \nabla^2  f_i(\bfx^\star) E_i 
	 \nabla^2 f_i(\bfx^\star) \nonumber \\
	&= \eta \nabla^2 f_i(\bfx^\star) (2 E_i-\eta E_i \nabla^2 f_i(\bfx^\star) E_i) 
	\nabla^2  f_i(\bfx^\star).
	\label{defhessV1}
\end{align}
From the Lipschitz continuity of $f_i(\bfx)$ we have that $\nabla^2 f_i(\bfx^\star) \preceq L_f I_{n}$.  
Substituting into~\eqref{defhessV1}, we obtain 
\[\begin{aligned}
\nabla^2 V_i(\bfx^\star;\eta) \succeq& 
 \eta \nabla^2 f_i(\bfx^\star) (2 E_i - (\eta L_f) E_i^2 ) \nabla^2  f_i(\bfx^\star) \\
  \succeq& \eta \nabla^2 f_i(\bfx^\star) E_i \nabla^2  f_i(\bfx^\star) 
\end{aligned}\]
where the final simplification follows from $\eta L_f \leq 1$ and $E_i^2 = E_i$.  The claim 
follows from the positive semidefiniteness of 
$\nabla^2 f_i(\bfx^\star) E_i \nabla^2  f_i(\bfx^\star)$.  Since $\nabla^2 V(\bfx^*;\eta)$ 
is the sum of positive semidefinite matrices the claim holds.
\end{proof}

\subsection{Convexity Properties of GNI: An Example}\label{sec:gniconvex}
In this section, we present an example NE reformulation of a (non-) convex game using the GNI setup. Suppose the player's objective is quadratic, \emph{i.e.,} 
$f_i(\bfx) = \half \bfx^T \bfQ_i \bfx + \bfr_i^T \bfx$.  Then, the GNI function is 
\begin{align}
& V_i(\bfx) = f_i(x) - f_i(\bfx - \eta E_i (\bfQ_i \bfx + \bfr_i)) \\
=& \half \bfx^T \left( \bfQ_i - \widehat{\bfQ}_i^T\bfQ_i \widehat{\bfQ}_i \right) \bfx 
					   + \eta \bfr_i^TE_i \bfQ_i(I + \widehat{\bfQ}_i) \bfx \nonumber \\
					   &\,  + \half \eta \bfr_i^T (2E_i - \eta E_i\bfQ_i E_i) \bfr_i \nonumber	
\end{align}
where $\widehat{\bfQ}_i = (I - \eta E_i\bfQ_i)$.  Suppose $\|\bfQ_i\| \leq L_f$ and let
$\eta \leq \frac{1}{L_f}$, then
\begin{equation}
	(\bfQ_i-\widehat{\bfQ}_i^T\bfQ_i \widehat{\bfQ}_i) 
	=\eta(\bfQ_iE_i) (2I - \eta \bfQ_i) (E_i \bfQ_i) \succeq 0
\end{equation}
where the positive semidefiniteness holds since for all $u \neq 0$ 
$u^T (\bfQ_iE_i) (2I - \eta \bfQ_i) (E_i \bfQ_i) u$ 
= $ (\bfQ_iE_i u)^T (2I - \eta \bfQ_i) (\bfQ_iE_i u) \geq 0$.  
Hence, when $f_i(\bfx)$ is quadratic, the  GNI function is a convex, quadratic function.  
Note that the convexity of GNI function holds regardless of the convexity of the original 
function $f_i(\bfx)$. However, for general nonlinear functions $f_i(\bfx)$, the GNI function $V_i(\bfx)$ does 
not preserve convexity. 

\section{Descent Algorithm for GNI}\label{sec:algorithm}
Consider the gradient descent iteration minimizing $V(\bfx;\eta)$ 
\begin{equation}
	\bfx^{k+1} = \bfx^k - \rho \nabla V(\bfx^k;\eta) \text{ for } k = 0,1,2,\ldots \label{gradDescent}
\end{equation}
where $\rho > 0$ is a stepsize. The restrictions on $\rho$, if any, are provided in 
subsequent discussions.

Theorem~\ref{thm:descViter} proves sublinear convergence of $\{\bfx^k\}$ 
to a stationary point of GNI function based on standard analysis. 
Linear convergence to a stationary point point is shown under the 
assumption of the  Polyak-{\L}ojasiewicz inequality~\cite{Lojasiewicz63,Polyak63,PLinequality}.  
\citet{LuoTseng93} employed similar error bound conditions in the context of 
descent algorithms of variational inequalities. 

\begin{theorem}\label{thm:descViter}
Suppose $\nabla V(\bfx)$ is $L_V$-Lipschitz continuous.  Let $\rho = \frac{\alpha}{L_V}$ for 
$0 < \alpha \leq 1$. Then, the $\{\bfx^k\}$ generated 
by~\eqref{gradDescent} converges sublinearly to $\bfx^\star$ a first-order stationary point of 
$V(\bfx;\eta)$, \emph{i.e.} $\nabla V(\bfx^\star;\eta) = 0$.  If $V(\bfx;\eta) \leq \frac{1}{2\mu} 
\|\nabla V(\bfx;\eta)\|^2$ then the sequence $\{V(\bfx^k)\}$ converges linearly to 0, \emph{i.e.,} 
$\{\bfx^k\}$ converges to $\bfx^\star \in {\cal S}^{SNP}$.
\end{theorem}

\begin{proof}
From Lipschitz continuity of $\nabla V(\bfx;\eta)$  
\begin{equation}
\begin{aligned}
V(\bfx^{k+1};\eta) \leq& V(\bfx^k;\eta) + \nabla V(\bfx^k;\eta)^T 
(\bfx^{k+1}-\bfx^k)  \\
&+ \frac{ L_V}{2} \|\bfx^{k+1}-\bfx^k\|^2 \\
\leq& V(\bfx^k,\eta) - \rho (1  - \frac{ \rho L_V }{2} ) \|\nabla V(\bfx;\eta)\|^2 \\
\leq& V(\bfx^k;\eta) - \frac{\overline{\alpha}}{2L_V}  \|\nabla V(\bfx;\eta)\|^2 
\end{aligned}\label{descVIter}
\end{equation}
where $\overline{\alpha} = \alpha(2-\alpha)$.  
Telescoping the sum for $k = 0,...,K$, we obtain
\begin{equation}
 V(\bfx^{K+1};\eta) \leq V(\bfx^0) -  \frac{\overline{\alpha}}{2L_V} \sum\limits_{k=0}^{K} 
 \|V(\bfx^K;\eta)\|^2.
 \end{equation}
 Since $V(\bfx;\eta)$ is bounded below by $0$, we have that 
 \[\begin{aligned}
&  \frac{\overline{\alpha}}{2L_V} \sum\limits_{k=0}^{K} 
 \| \nabla V(\bfx^K;\eta)\|^2 \leq  V(\bfx^0) - V(\bfx^{K+1}) \leq V(\bfx^0) \\
 & \implies \frac{\overline{\alpha}}{2L_V} \min\limits_{k \in \{0,\ldots,K\}} \| \nabla V(\bfx^k;\eta)\|^2 
 \leq \frac{V(\bfx^0)}{K+1}.
\end{aligned} \]
This proves the claim on sublinear convergence to a first-order stationary point of $V(\bfx;\eta)$.  
Suppose $V(\bfx;\eta) \leq \frac{1}{2\mu} \|\nabla V(\bfx;\eta)\|^2$ holds.  Substituting 
in~\eqref{descVIter} obtain
\begin{equation}
\begin{aligned}
V(\bfx^{k+1};\eta) \leq& \left( 1 -  \overline{\alpha}\frac{\mu}{L_V} \right)V(\bfx^k;\eta)
\end{aligned}\label{descVIter1}
\end{equation}
which proves the claim on linear convergence of $\{V(\bfx^k;\eta)\}$ to 0.  By Theorem~\ref{thm:zeroVEqForderCrit}, $\{\bfx^k\}$ converges to 
$\bfx^\star \in {\cal S}^{SNP}$.
\end{proof}

\subsection{Quadratic Objectives}\label{sec:quadobjs}
In the following, we explore a popular setting of quadratic objective function and explore the 
implication of Theorem 2.   Note that the bilinear case is a special case of the quadratic objective.  
Consider the $f_i(\bfx)$'s to be quadratic.  For this setting  \S\ref{sec:gniconvex} showed that 
GNI function $V_i(\bfx)$ is a convex quadratic function.      
	This proves that $V_i(\bfx;\eta)$ has  
	 $(3L_f)$-Lipschitz continuous gradient.  It is well known that  
	for a composition of a linear function with a strongly convex function, we have that 
	Polyak-{\L}ojasiewicz inequality holds~\cite{LuoTseng93}, i.e., there exists $\mu > 0$ such that 
	$V(\bfx;\eta) \leq \frac{1}{2\mu} \|\nabla V(\bfx;\eta)\|^2$ holds.
    Hence, we can state the following stronger result for quadratic 
    objective functions.
\begin{corollary}
Suppose $f_i(\bfx)$ are quadratic and player convex, i.e. $f_i(\bfx)$ is 
convex in $x_i$.  Let $\rho = \frac{1}{3 L_f N}$. 
Then, the sequence $\{V(\bfx^k)\}$ converges linearly to 0, \emph{i.e.} 
$\{\bfx^k\}$ converges to $\bfx^\star \in {\cal S}^{NE}$.
\end{corollary}

\section{Modified Descent Algorithm for GNI}\label{sec:moddescent}
The evaluation of the gradient $\nabla V(\bfx;\eta)$ requires the computation of the Hessian of the 
functions $f_i(\bfx)$ (see~\eqref{defgradV}) which can be prohibitive to compute.  
A close examination of the expression of $\nabla V(\bfx;\eta)$ in~\eqref{defgradV} 
reveals that we only require the action of the Hessian in a particular direction, i.e. 
$\nabla f_i(\bfx -\eta E_i \nabla f_i(\bfx))$.  This immediately suggests the 
use of an approximation for this term inspired by secant  
methods~\cite{NocWrightbook}
\begin{align}
&\nabla^2 f_i(\bfx) (\eta E_i \nabla f_i(\bfx - \eta E_i\nabla f_i(\bfx))) \nonumber \\
	\approx&\; \nabla f_i( \bfx + \eta E_i \nabla 
	f_i(\bfx - \eta E_i\nabla f_i(\bfx)) ) - \nabla f_i(\bfx) 
	\label{QNapprox}
\end{align}
Substituting~\eqref{QNapprox} for the term involving the Hessian in 
$\nabla V_i(\bfx;\eta)$ and simplifying obtain the direction $\nabla \widehat{V}_i(\bfx;\eta)$:
\begin{align}
    \label{modifiedGradV}
	\nabla \widehat{V}_i(\bfx;\eta) =&\; \nabla f_i(\bfx + \eta E_i \nabla 
	f_i(\bfx - \eta E_i\nabla f_i(\bfx)) ) \nonumber \\
	& - \nabla f_i(\bfx - \eta E_i \nabla f_i(\bfx)) 
\end{align}
Substituting~\eqref{QNapprox} in the gradient descent iteration~\eqref{descVIter}, we obtain the modified iteration
\begin{equation}
	\label{moddescViter}
	\bfx^{k+1} = \bfx^k - \rho \nabla \widehat{V}_i (\bfx;\eta) 
	\text{ for } k = 0,1,2,\ldots
\end{equation} 
where $\nabla \widehat{V}_i (\bfx;\eta) = \sum\limits_{i=1}^N \nabla \widehat{V}_i(\bfx;\eta)$.
We assume that the following bound on the error in the 
approximation
\begin{equation}
	\|\nabla \widehat{V}(\bfx;\eta) - \nabla V(\bfx;\eta)\| \leq \tau \|\nabla V(\bfx;\eta)\|,
	\label{bndmodgrad}
\end{equation}
for some $\tau \in (0,1)$.  
Such a bound on the error in the gradients has also been used in~\citet{LuoTseng93}.

\begin{theorem}\label{thm:moddescViter}
Suppose $\nabla V(\bfx)$ is $L_V$-Lipschitz continuous.  Let $\rho =\alpha \frac{1-\tau}{L_V(1+\tau)^2}$ for $0 < \alpha \leq 1$ and~\eqref{bndmodgrad} holds. Then, the $\{\bfx^k\}$ generated 
by~\eqref{moddescViter} converges sublinearly to $\bfx^\star$ a first-order stationary point of 
$V(\bfx;\eta)$, \emph{i.e.,} $\nabla V(\bfx^\star;\eta) = 0$.  If $V(\bfx;\eta) \leq \frac{1}{2\mu} 
\|\nabla V(\bfx;\eta)\|^2$, then the sequence $\{V(\bfx^k)\}$ converges linearly to 0, \emph{i.e.,} 
$\{\bfx^k\}$ converges to $\bfx^\star \in {\cal S}^{SNP}$.
\end{theorem}
\begin{proof}
Let $\nabla \widehat{V}(\bfx^k;\eta) = \nabla V(\bfx^k;\eta) + e^k$.  From~\eqref{bndmodgrad}, 
$\|e^k\| \leq \tau \|\nabla V(\bfx^k;\eta)\|$.  
Applying the triangle inequality to 
$\|\nabla \widehat{V}(\bfx^k;\eta)\|$ and use~\eqref{bndmodgrad} obtain
\begin{equation}\begin{aligned}
\|\nabla \widehat{V}(\bfx^k)\| &\leq \|\nabla V(\bfx^k;\eta) + \|e^k\| \\
& \leq (1+\tau)\|\nabla V(\bfx^k;\eta)\|. 
\end{aligned}\label{bndgradwidehatV2}\end{equation}
The term $-(\nabla V(\bfx^k;\eta))^T (\nabla \widehat{V}(\bfx^k;\eta))$ can be upper 
bounded as
\begin{equation}\begin{aligned}
& -(\nabla V(\bfx^k;\eta))^T (\nabla \widehat{V}(\bfx^k;\eta)) \\
=& - \|\nabla V(\bfx^k;\eta)\|^2 - (\nabla V(\bfx^k;\eta))^T e^k \\
\leq& - \|\nabla V(\bfx^k;\eta)\|^2 + \|\nabla V(\bfx^k;\eta)\| \|e^k\| \\
\leq& -(1-\tau) \|\nabla V(\bfx^k;\eta)\|^2
\end{aligned}\label{bndcrossterm}\end{equation}
where the final inequality follows from~\eqref{bndmodgrad}.  
From Lipschitz continuity of $\nabla V(\bfx;\eta)$  
\begin{equation}
\begin{aligned}
& V(\bfx^{k+1};\eta) \\
\leq& V(\bfx^k;\eta) + \nabla V(\bfx^k;\eta)^T (\bfx^{k+1}-\bfx^k)  \\
&+ \frac{ L_V}{2} \|\bfx^{k+1}-\bfx^k\|^2 \\
\leq& V(\bfx^k,\eta) - \rho (\nabla V(\bfx^k;\eta))^T(\nabla \widehat{V}(\bfx^k;\eta)) + \\
& \frac{ L_V \rho^2}{2} \|\nabla \widehat{V}(\bfx;\eta)\|^2 \\
\leq& V(\bfx^k,\eta) - \rho (1-\tau) \|\nabla V(\bfx^k;\eta)\|^2 + \\
& \frac{ L_V \rho^2}{2} (1+ \tau)^2 \|\nabla V(\bfx;\eta)\|^2 \\
\leq& V(\bfx^k;\eta) -
  \rho \left( 1 - \tau - \frac{L_V \rho(1+\tau)^2}{2} \right) \|\nabla V(\bfx;\eta)\|^2 \\
\leq& V(\bfx^k;\eta) - \frac{\overline{\alpha}}{2} \frac{(1-\tau)^2}{ L_V (1+ \tau)^2}  \|\nabla V(\bfx;\eta)\|^2
\end{aligned}\label{descVIterm}
\end{equation}
where $\overline{\alpha} = \alpha (2-\alpha)$, 
the third inequality is obtained by  
substituting~\eqref{bndgradwidehatV2} and~\eqref{bndcrossterm}, and the 
final inequality follows from the definition of $\rho$ in the 
statement of the theorem.  
By similar arguments to those in Theorem~\ref{thm:descViter} obtain
 \[\begin{aligned}
 & \overline{\alpha} \frac{(1-\tau)^2}{2 L_V (1+ \tau)^2} \min\limits_{k \in \{0,\ldots,K\}} \| \nabla V(\bfx^k;\eta)\|^2 
 \leq \frac{V(\bfx^0)}{K+1}.
\end{aligned} \]
This proves the claim on sublinear convergence to a first-order stationary point of $V(\bfx;\eta)$.  
Suppose $V(\bfx;\eta) \leq \frac{1}{2\mu} \|\nabla V(\bfx;\eta)\|^2$ holds.  Substituting 
in~\eqref{descVIterm} obtain
\begin{equation}
\begin{aligned}
V(\bfx^{k+1};\eta) \leq& \left( 1 - \overline{\alpha} \frac{\mu (1-\tau)^2}{ L_V (1+ \tau)^2} \right)V(\bfx^k;\eta)
\end{aligned}\label{descVIterm1}
\end{equation}
which proves the claim on linear convergence of $\{V(\bfx^k;\eta)\}$ to 0.  By 
Theorem~\ref{thm:zeroVEqForderCrit}, $\{\bfx^k\}$ converges to 
$\bfx^\star \in {\cal S}^{SNP}$.
\end{proof}
The approximation in~\eqref{QNapprox} is in fact exact when the function $f_i(x)$ is quadratic. Consequently, the claims on the convergence of the iterates continue to hold when the iterates are generated 
by~\eqref{moddescViter}.






\section{Experiments}\label{sec:expts}
In this section, we present several empirical results on simulated data demonstrating the effectiveness of the proposed GNI formulation. To demonstrate the correctness of our theoretical results, we show numerical results on several simple game settings with known equilibrium. Specifically, we consider the following payoff functions: i) bilinear two-player games, ii) quadratic games with convex and non-convex payoffs, iii) linear GAN using a Dirac delta generator, and iv) a more general linear GAN with linear generator and discriminator. We compare our descent algorithm against several popular choices such as (i) gradient descent, (ii) gradient descent with Adam-style updates~\cite{kingma2014adam}, (iii) optimistic mirror descent~\cite{RakhlinSridharan,DaskalakisGAN1}, (iv) the extrapolation scheme~\cite{VIGAN}, and (v) the extra-gradient method~\cite{korpelevich1976extragradient}. For all these methods, we either follow the standard hyperparameter settings (e.g., in Adam), or we find the hyperparameters that lead to the best convergence. 
 For each of these games, we observe convergence of the proposed algorithm to stationary Nash points and contrast the quality of solutions against what can be theoretically guaranteed. As discussed in Section~\ref{sec:gniconvex}, the quadratic and bilinear cases lead to convex GNI function and thus, the game always converges to a NE. Refer to supplementary materials for extra experiments. Below, we detail each of the game settings. 
\subsection{Bi-Linear Two-player Game:} 
We consider the following two-player game:
\begin{equation}
	f_1(x) = x_1^TQx_2 + q_1^T x_1 + q_2^T x_2 = -f_2(x) \label{eq:bilinear},
\end{equation}
where $f_1$ and $f_2$ are the player's payoff functions -- a setting explored in~\cite{VIGAN}. The GNI for this game leads to a convex objective. For GNI, we use a step-size $\eta=1/L$, where $L=\norm{Q}$, and $\rho=0.01$, while for other methods we use a stepsize of $\eta=0.001$\footnote{Other values of $\eta$ did not seem to result in stable descent.}. The methods are initialized randomly -- the initialization is seen to have little impact on the convergence of GNI, however changed drastically for that of others.

\begin{figure}
    \centering
    \subfigure[]{\includegraphics[width=4cm]{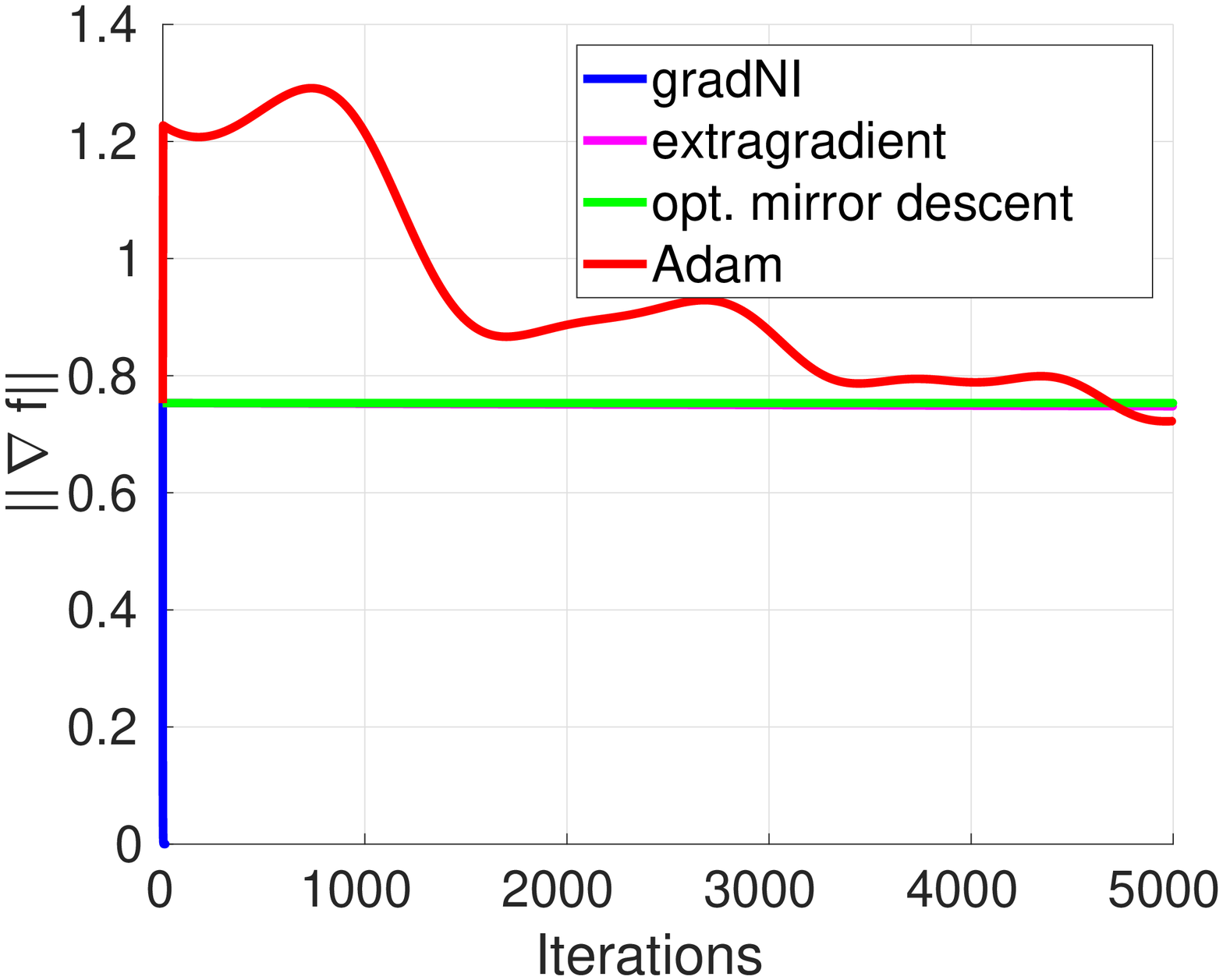}}
    \subfigure[]{\includegraphics[width=4cm]{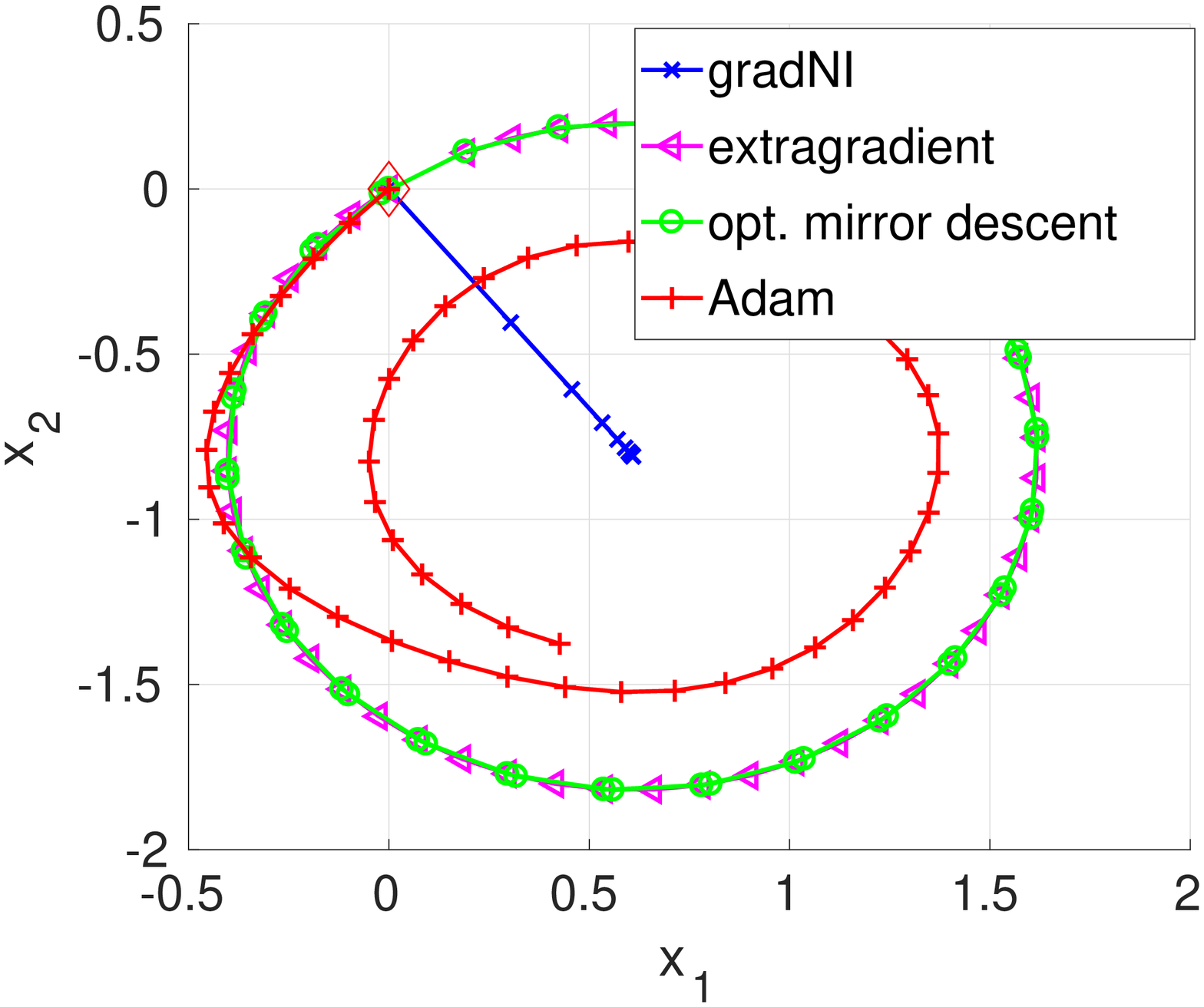}}
    \caption{(a) shows GNI against other methods for bilinear min-max game.~(b) shows convergence trajectories for 1-dimensional players. For (b), the initial point is shown in red diamond.}
    \label{fig:bilinear}
\end{figure}
In Figure~\ref{fig:bilinear}(a), we plot the gradient convergence (using 10-d data). In this plot (and all subsequent plots of gradient 
convergence), the norm of the gradient 
$\| \nabla f (\bfx^k) \| = \|(\nabla_1 f_1(\bfx^k),\ldots,\nabla_N f_N(\bfx^k)) \|$.
We see that GNI converges linearly. However, other methods, such as gradient descent and mirror descent iterates diverge, while the extragradient and Adam are seen to converge slowly. To understand the descent better, in Figure~\ref{fig:bilinear}(b), we use $x_1,x_2\in\R^1$, and plot them for every 100-th iteration starting from the same initial point (shown by the red-diamond). Interestingly, we find that the extragradient and mirror-descent methods show a circular trajectory, while Adam (with $\beta_1=0.9$ and $\beta_2=0.999$) takes a spiral convergence path. GNI takes a more straight trajectory steadily decreasing to optima (shown by the blue straight line).

\begin{figure}[!ht]
    \centering
    \subfigure[non-convex QP]{\includegraphics[width=4cm]{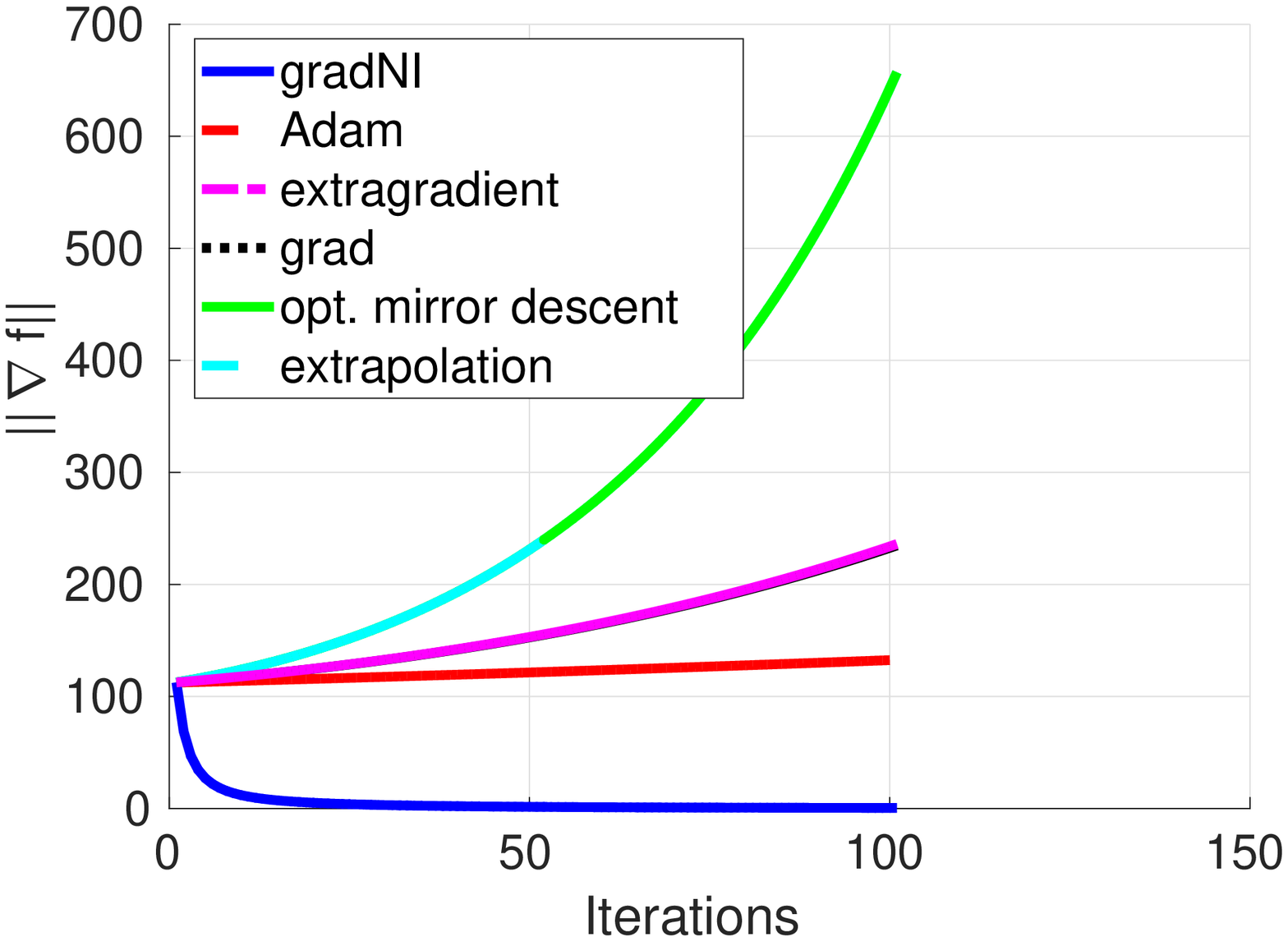}}
    \subfigure[convex-QP]{\includegraphics[width=4cm]{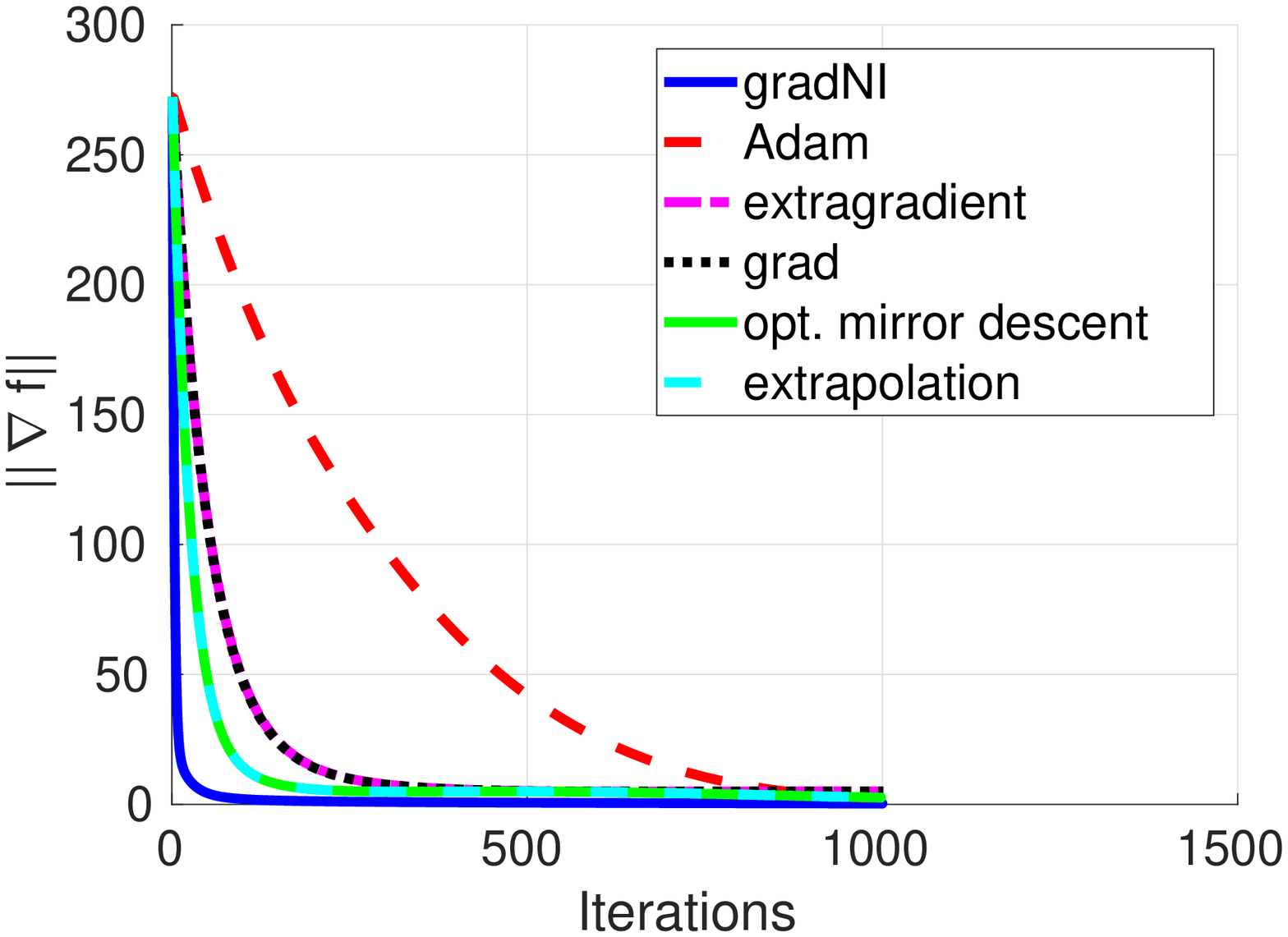}}
    \caption{Convergence of GNI against other methods for Quadratic games. (a) Non-convex QP with indefinite Q matrices for each player, (b) convex QP with semi-definite Q matrices.}
    \label{fig:quad_semidefinite}
\end{figure}
\subsection{Two-Player Quadratic Games:} 
We consider two-player games (multiplayer extensions are trivial) with the payoff functions:
\begin{equation}
f_i(x) = \half x^TQ_ix + r_i^T x\text{, for } i = 1,2 \label{eq:quadObj}
\end{equation}
where $Q_i \in \R^{n \times n}$ is symmetric. We consider cases when each $Q_i$ is indefinite (i.e., non-convex QP) and positive semi-definite. As with the bilinear case, all the QP payoffs result in convex GNI reformulations. We used 20-d data, the same stepsizes $\eta=\max_i(\norm{Q_i})$ and $\rho=0.01$ for GNI, while using $\eta=10^{-4}$ for other methods. The players are initialized from $N(0,I)$.  

In Figure~\ref{fig:quad_semidefinite}, we compare the descent on these quadratic games. We find that the competitive methods are difficult to optimize for the non-convex QP and almost all of them diverge, except Adam which converges slowly. GNI is found to converge to the stationary Nash point (as it is convex-- in~\S\ref{sec:gniconvex}). For the convex case, all methods are found to converge. To gain insight, we plot the convergence trajectory for a 1-d convex quadratic game (i.e., $x_1,x_2\in \R^1$) in Figure~\ref{fig:quad_1d_convergence}. The initializations are random for both players and the parameters are equal. We see that all schemes follow similar trajectories, except for Adam and GNI -- all converging to the same point.  

\begin{figure}[!h]
    \centering
    \includegraphics[width=4cm]{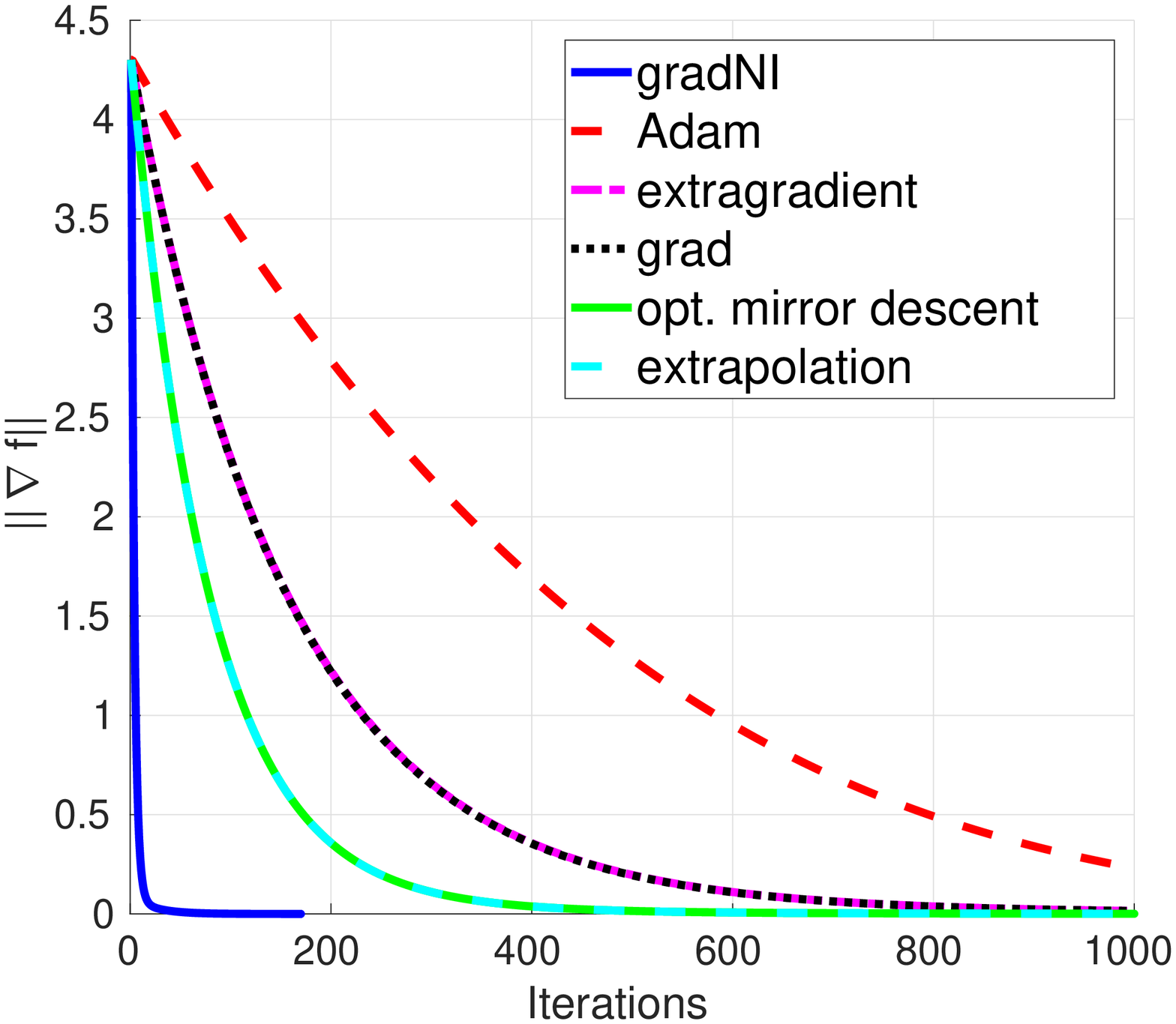}
    \includegraphics[width=4cm]{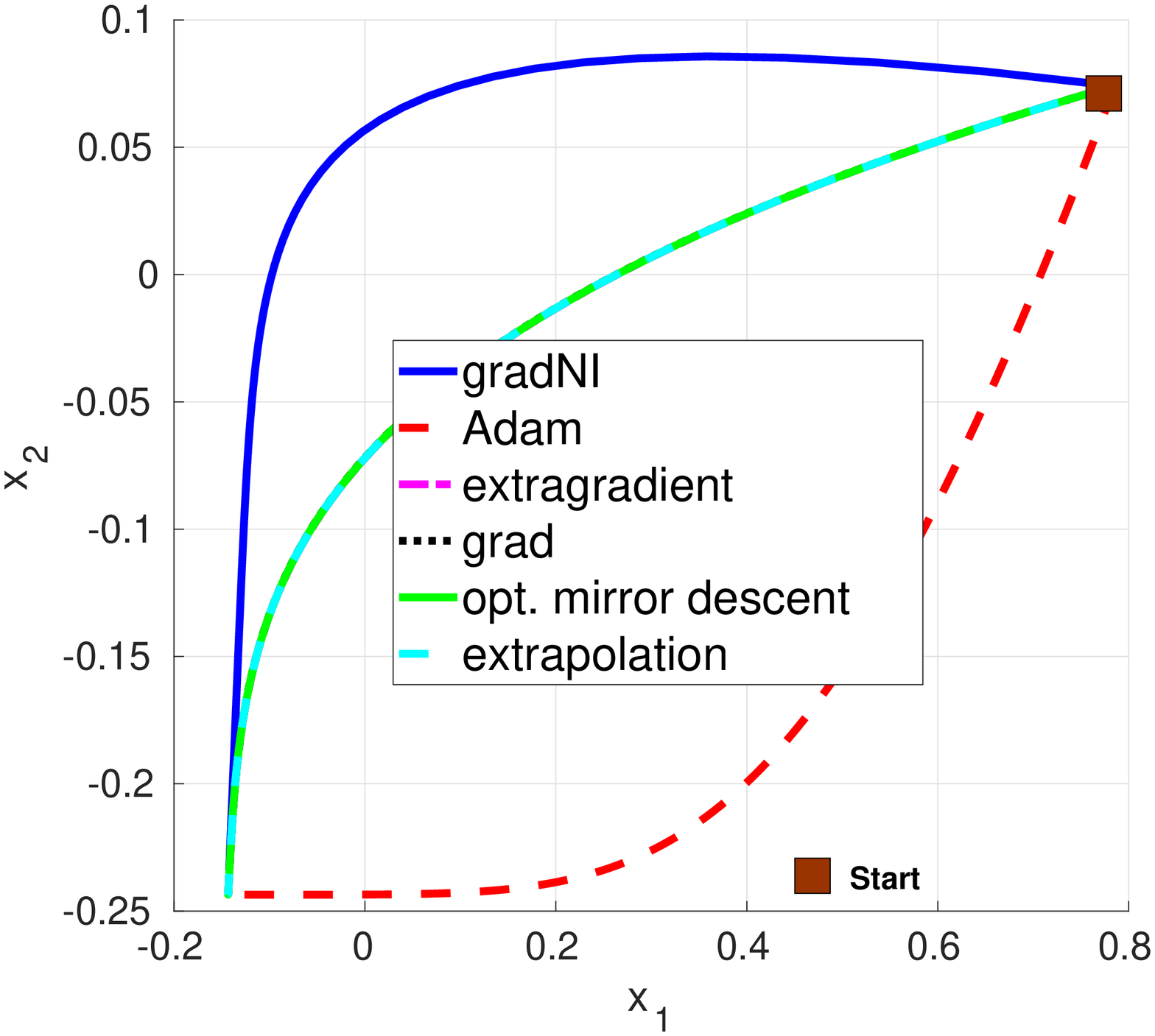}
    \caption{Convergence of GNI against other methods on a convex $1-$d quadratic game.  Left: the convergence achieved by different algorithms. Right: the trajectories of the two players to the NE.}
    \label{fig:quad_1d_convergence}
\end{figure}

\subsection{Dirac Delta GAN}
This is a one-dimensional GAN explored in~\cite{VIGAN}. In this case, the real data is assumed to follow a Dirac delta distribution (with a spike at say point -2). The payoff functions for the two players are:
\begin{align}
    f_1 =& \log(1+\exp(\theta x_1)) + \log(1+\exp(x_1 x_2))\nonumber\\
    f_2 =& -\log(1+\exp(x_1x_2)),
    \label{eq:deltagan}
\end{align}
where $\theta\in\R^1$ is the location of the delta spike. Unlike other game settings described above, we do not have an analytical formula to find the Lipscitz constant for the payoffs. To this end, we did an empirical estimate (more details to follow). We used $L=2$, $\eta=\rho=1/L$ and initialized all players uniformly from $[0,4]$.

Figure~\ref{fig:dirac-delta-gan} shows the comparison of the convergence of the dirac delta GAN game to a stationary Nash point. The GNI achieves faster convergence than all other methods, albeit having a non-convex reformulation in contrast to the bilinear and QP cases discussed above. The game has multiple local solutions and the schemes may converge to varied points depending on their initialization (see supplementary material for details). 

\begin{figure}
    \centering
    \includegraphics[width=4.5cm]{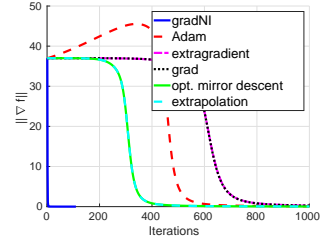}
    \caption{Convergence of GNI against other methods on the Dirac-Delta GAN.}
    \label{fig:dirac-delta-gan}
\end{figure}

\begin{figure*}[!h]
    \centering
    \subfigure[Varying $\eta$, $\rho=1$]{\includegraphics[width=4.5cm]{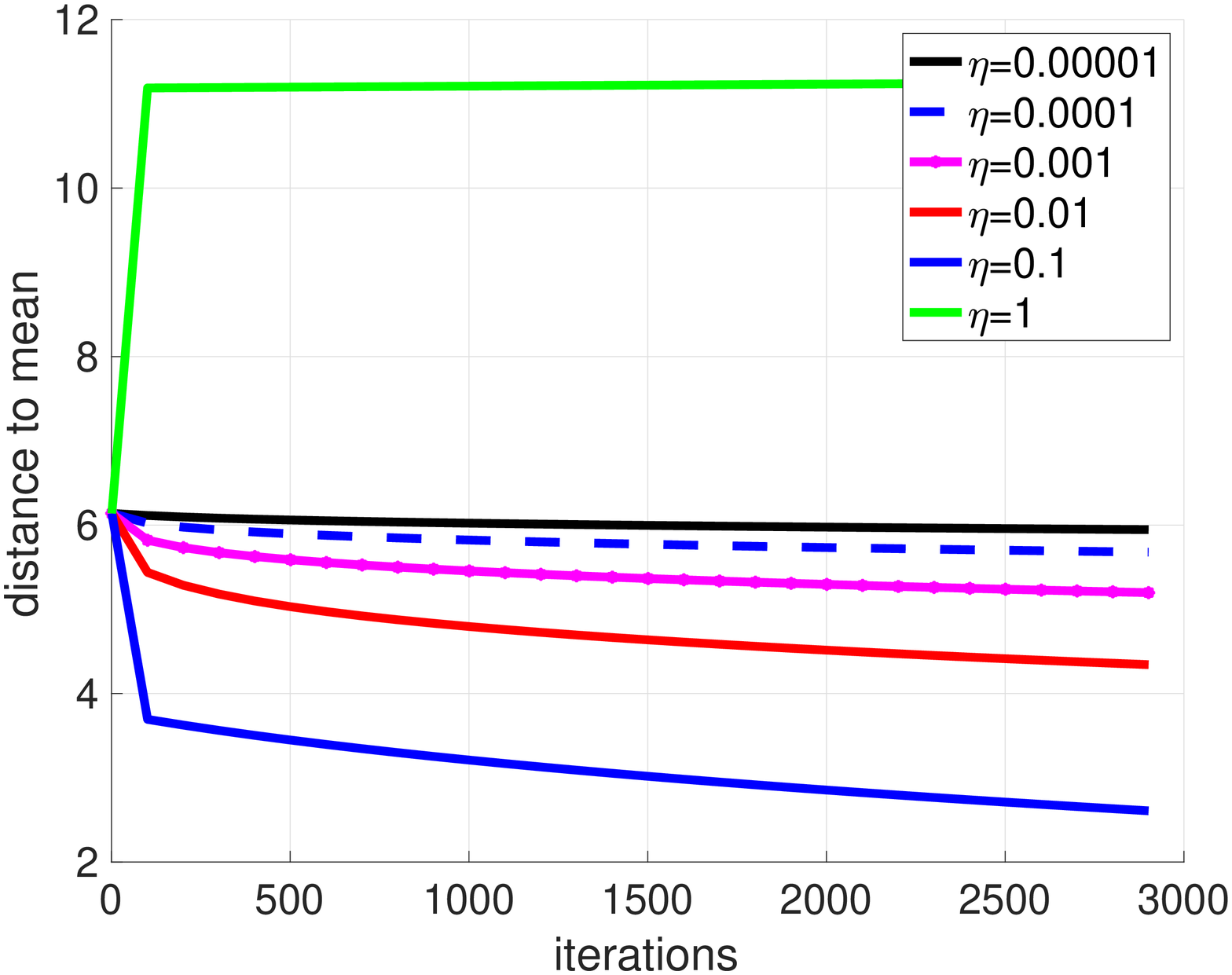}}
    \subfigure[varying $\eta$, $\rho=1$ ]{\includegraphics[width=4.5cm]{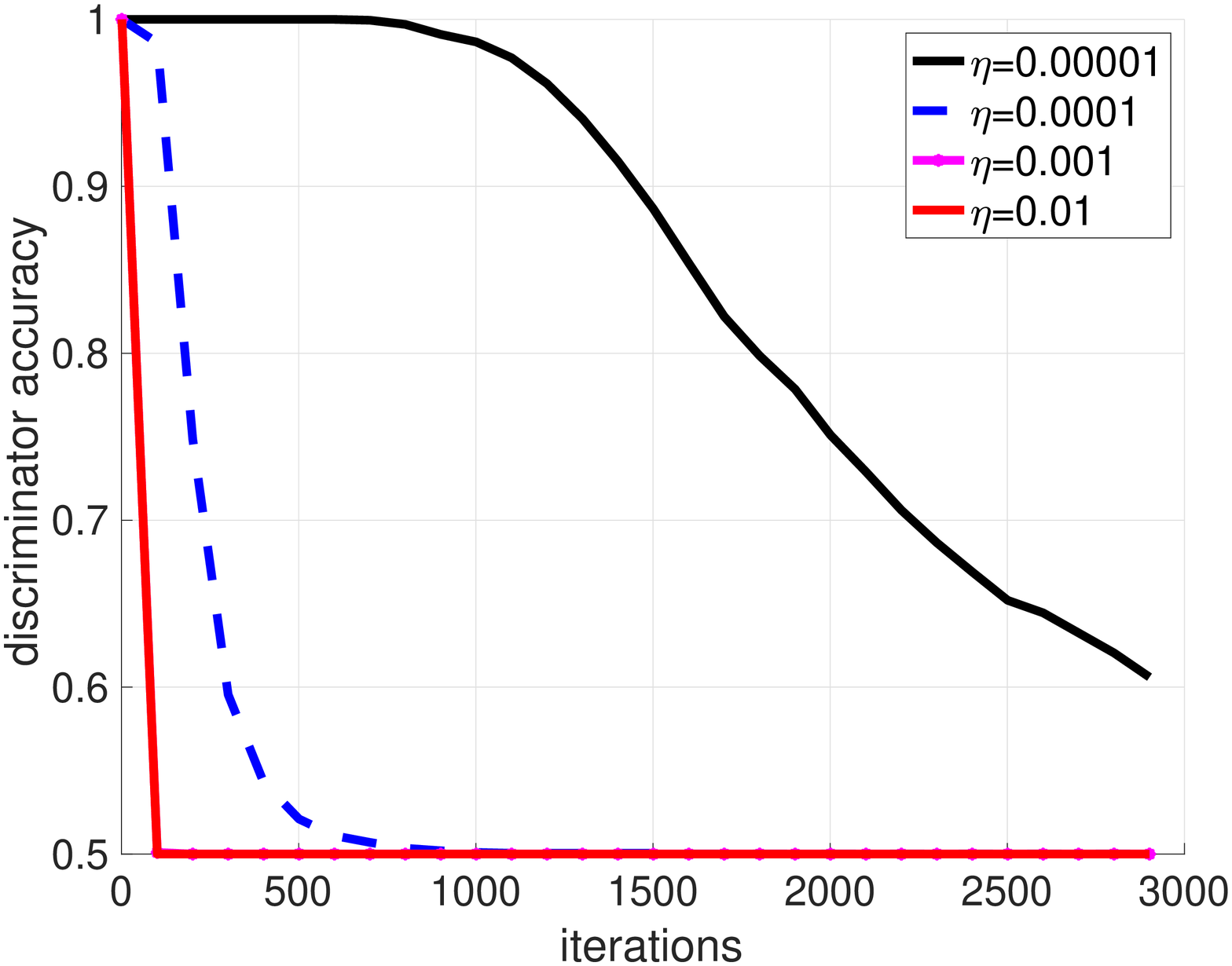}}
    \subfigure[varying $\rho$, $\eta=0.1$ ]{\includegraphics[width=4.5cm]{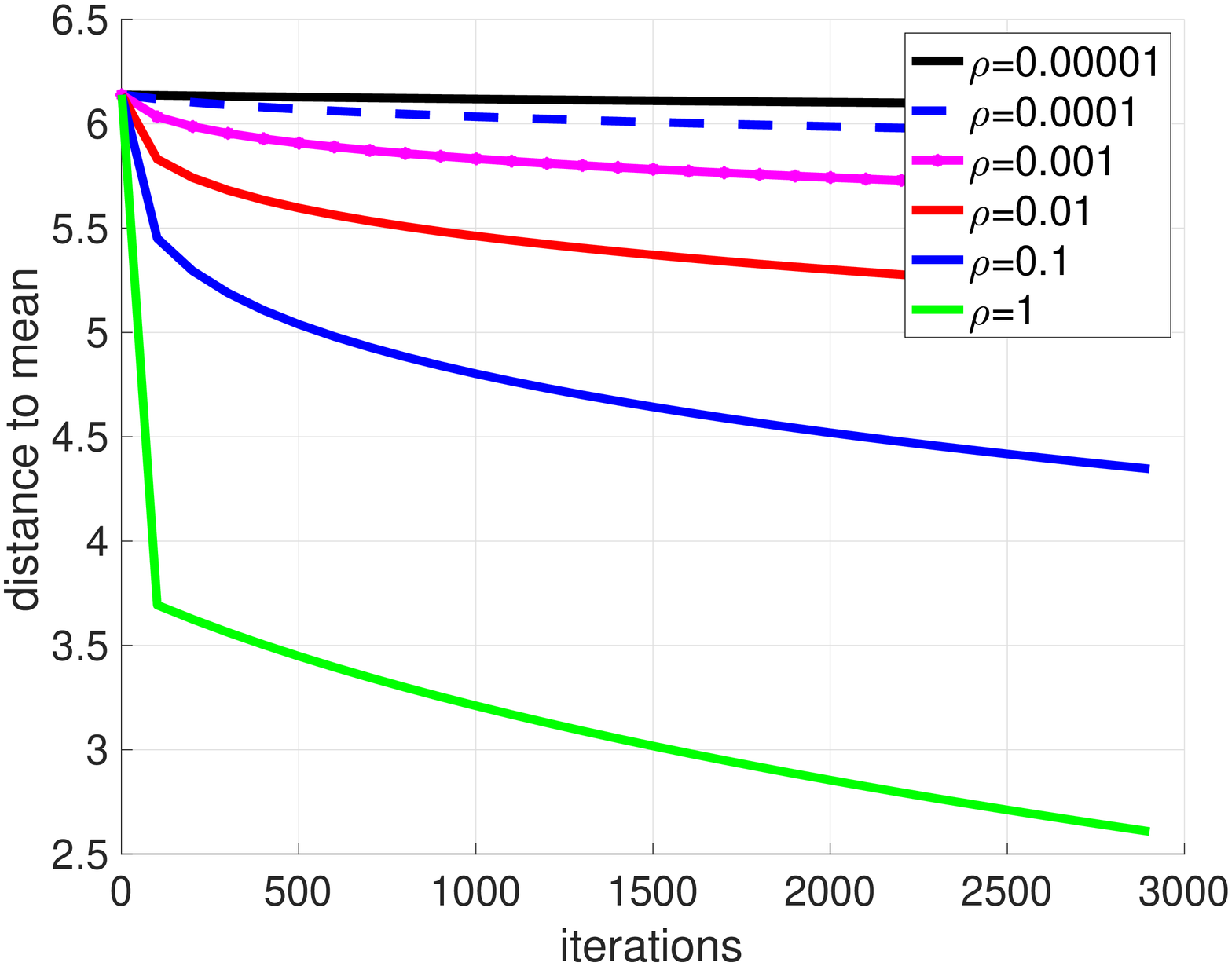}}
    \caption{Study of the influence of the step sizes ($\rho$ and $\eta$) on the convergence of GNI reformulations for the linear GAN game.}
    \label{fig:my_dcgan_parameter_study}
\end{figure*}
\begin{figure*}[!h]
    \centering
    \subfigure[Generator, $P_r=N(\mu,I)$]{\includegraphics[width=4.5cm]{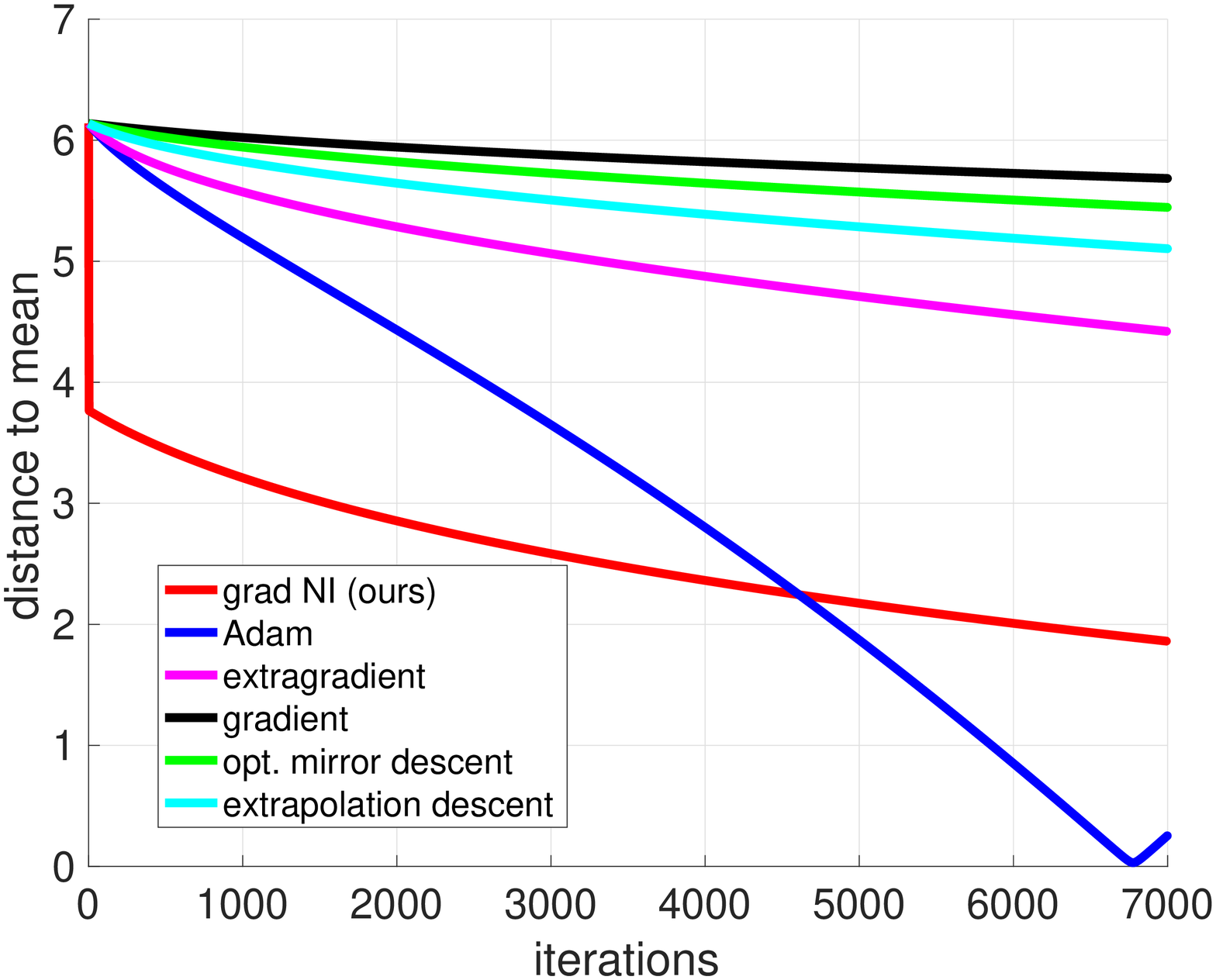}}
    \subfigure[Discriminator, $P_r=N(\mu,I)$]{\includegraphics[width=4.5cm]{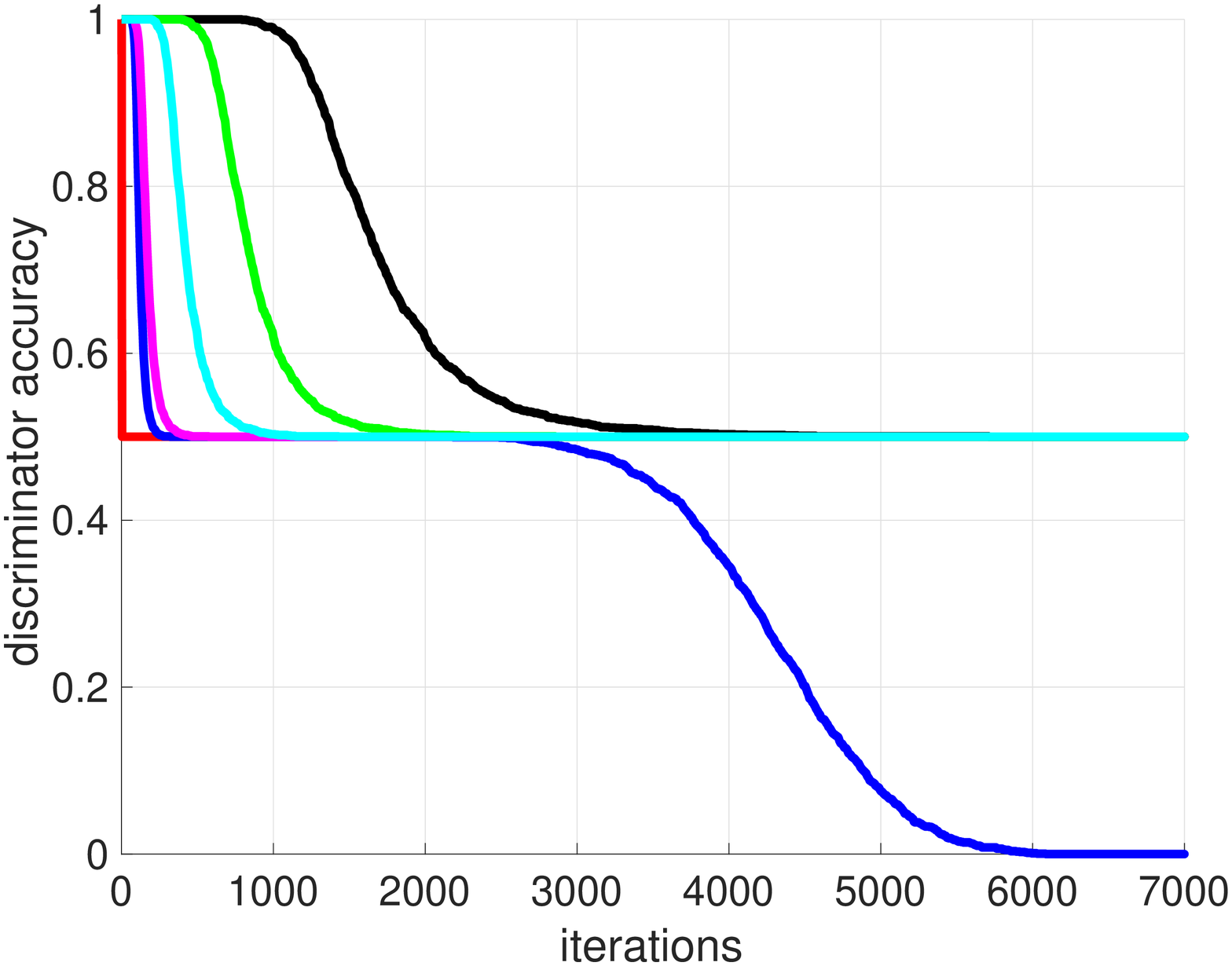}}
    \subfigure[Generator, $P_r=N(\mu,\Sigma)$]{\includegraphics[width=4.5cm]{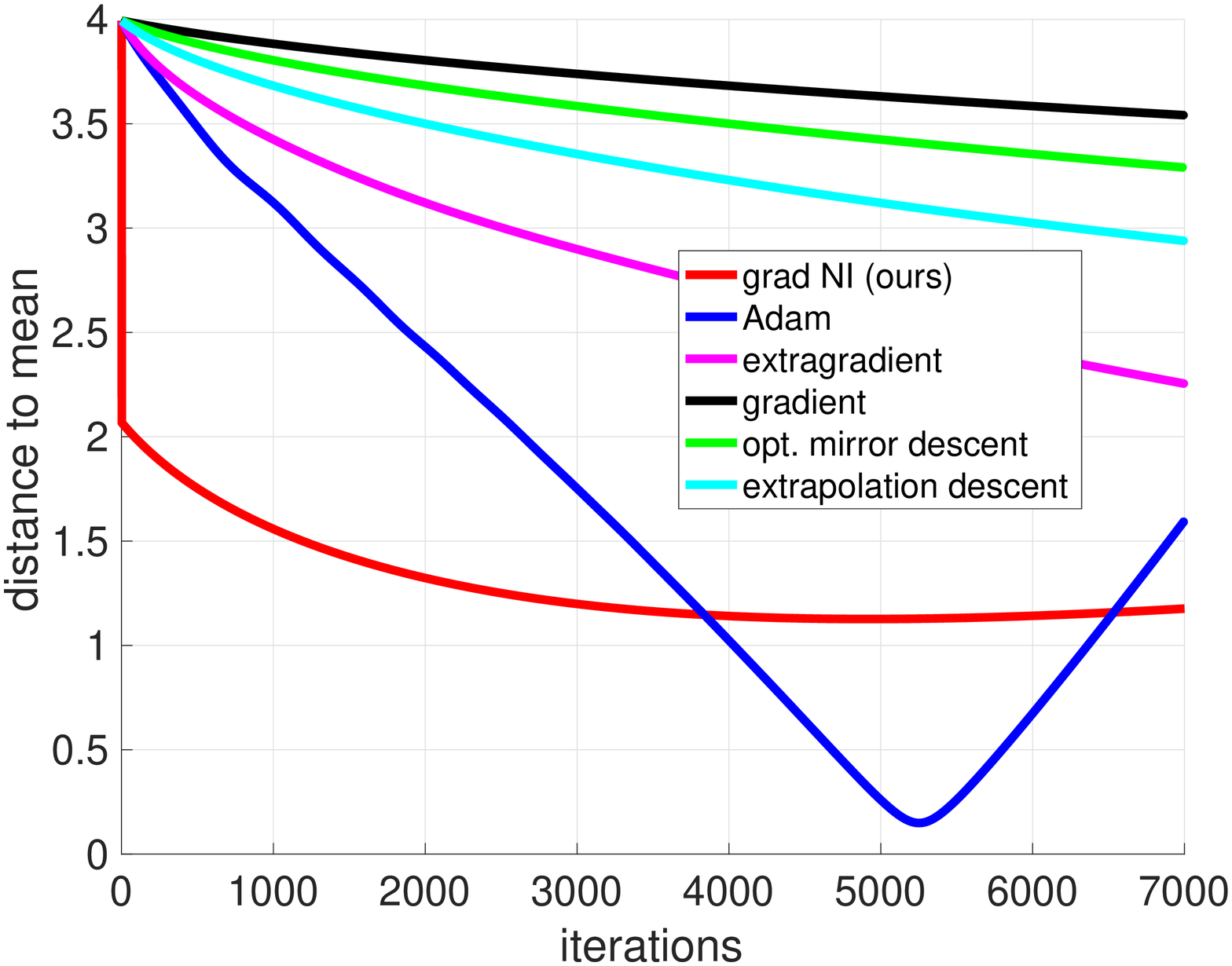}}
    \caption{Convergence of GNI against other methods on the linear GAN two-player game. The real-data distribution is sampled from $N(\mu,I)$ for (a) and (b), while we use $N(\mu,\Sigma)$ for (c), where $\Sigma=\diag(\xi),\xi\sim U(0,1]$. Note that, when the optimization converges, the discriminator is expected to be confused between the real and fake data distributions (i.e., classification accuracy is 0.5).}
    \label{fig:dcgan_convergence}
\end{figure*}

\subsection{Linear GAN}
We now introduce a more general GAN setup -- a variant of the non-saturating GAN described in~\cite{Goodfellow16}, however using a linear generator and discriminator. We designed this experiment to serve two key goals: (i) to exposit the influence of the GNI hyperparameters in a more general GAN setting, and (ii) show the performance of GNI on a setting for which it is harder to estimate a Lipschitz constant $L$. While, our proposed setting is not a neural network, it allows to understand the behavior of GNI when other non-linearities arising from the layers of a neural network are absent, and thereby study GNI in isolation. 

\noindent\paragraph*{Experimental Setup:} The payoff functions are:
\begin{align}
    f_1 =& -\!\expect_{\theta\sim \prob_r}\!\!\log\left(x_1^T\theta\right)\!-\!\expect_{z\sim P_z}\log\left(1-x_1^T\diag\left(x_2\right)z\right),\!\!\nonumber\\
    f_2 =& -\!\expect_{z\sim P_z}\log\left(x_2^T\diag\left(x_1\right)z\right),\label{eq:lineargan}
\end{align}
where $\prob_r$ and $\prob_z$ are the real and the noise data distributions, the latter being the standard normal distribution $N(0,I)$. The operator $\diag$ returns a diagonal matrix with its argument as its diagonal. We consider two cases for $\prob_r$: (i) $\prob_r=N(\mu,I)$ for a mean $\mu$ and (ii) $\prob_r=N(\mu,\Sigma)$ for a covariance matrix $\Sigma\in\R^{d\times d}$. In our experiments to follow, we use $\mu=2e$, $e$ being a $d$-dimensional vector ($d=10$) of all ones. We initialized $x_1=x_2=e/d$ for all the methods.

\noindent\paragraph*{Evaluation Metrics:} To evaluate the performance on various hyper-parameters of GNI, we define two metrics: (i) \emph{discriminator-accuracy}, and (ii) the \emph{distance-to-mean}.  The discriminator-accuracy measures how well the learned discriminator classifies the two distributions, defined as:
\begin{equation}
    \label{eq:dis_acc}
    \dist{acc}\!\!=\!\!\frac{1}{2M}\!\!\sum_{i=1}^M \mathcal{I}(x_1^T\theta_i \geq \zeta) + \mathcal{I}\left(x_1^T\!\diag(x^i_2)z_i\leq (1-\zeta)\right),\notag
\end{equation}
where $\mathcal{I}$ is the indicator function, $M$ is the number of data points sampled from the respective distributions, and $\zeta\in[0,1]$ is a threshold for the indicator function. We use $\zeta=0.7$. While $\dist{acc}$ measures the quality of the discriminator learned, it does not tell us anything on the convergence of the generator. To this end, we present another measure to evaluate the generator; specifically, the~\emph{distance-to-mean}, that computes the distance of the generated distribution from the first moment of the true distribution, defined as:
\begin{equation}
\dist{mean} = \norm{\expect_{z\sim \prob_z} {\diag(x_2)z} - \expect_{\theta\sim\prob_r} \theta} 
\end{equation}

\noindent\paragraph*{Hyper-parameter Study:} The goal of this experiment is to analyze the descent trajectory of GNI-based gradient descent when the hyper-parameters are changed. To this end, we vary $\eta$ and $\rho$ separately in the range $10^{-5}$ to $10$ in multiples of $10$, while keeping the other parameter fixed (we use $\eta=0.1$ and $\rho=1$ as the base settings). In Figure~\ref{fig:my_dcgan_parameter_study}, we plot the discriminator-accuracy and distance-to-mean against GNI iterations for the generator and discriminator separately. From Figures~\ref{fig:my_dcgan_parameter_study}(a) and (b), it appears that higher value of $\eta$ biases the descents on the generator and discriminator separately. For example, $\eta\geq 0.01$ leads to a sharp descent to the optimal solution of the discriminator, however, $\eta>1$ leads to a generator breakdown (Figure~\ref{fig:my_dcgan_parameter_study}(a)). Similarly, a small value of $\rho$, such as $\rho<10^{-5}$ shows high distance-to-mean, i.e., generator is weak, while, $\rho=1$ leads to good descents for both the generator and the discriminator. We found that a higher $\rho$ leads to unstable descent, skewing the plots and thus not shown. In short, we found that making the discriminator quickly converge to its optimum could lead to a better convergence trajectory for the generator for this linear GAN setup using the GNI scheme.

\noindent\paragraph*{Comparisons to Other Algorithms:}
In Figures~\ref{fig:dcgan_convergence}(a) and (b), we plot the distance-to-mean and discriminator-accuracy of linear GAN using $\eta=0.1$ and $\rho=1$, and compare it to all other descent schemes. Interestingly, we found that Adam shows a different pattern of convergence, with the distance-to-mean steadily decreasing to zero; on close inspection (Figure~\ref{fig:dcgan_convergence}(b)), we see that the discriminator-accuracy simultaneously goes to zero as well, suggesting the non-optimality of the descent. In contrast, our GNI converges quickly. In Figure~\ref{fig:dcgan_convergence}(c), we plot the convergence when using a real data distribution $P_r=N(\mu,\Sigma)$, where $\mu\sim U(0,1)^d$; a d-dimensional uniform distribution and $\Sigma$ is a randomly-sampled diagonal covariance matrix. The descent in this general setting also looks similar to the one in Figure~\ref{fig:dcgan_convergence}(a). 

\section{Conclusions}
We presented a novel formulation for Nash equilibrium computation in multi-player games by introducing the Gradient-based Nikaido-Isoda (GNI) function. The GNI formulation for games allows individual players to locally improve their objectives using steepest descent while preserving local stability and convergence guarantees. We showed that the GNI function is a valid merit function for multi-player games and presented an approximate descent algorithm. We compared our method against several popular descent schemes on multiple game settings and empirically demonstrated that our method outperforms all other techniques. Future research will explore the GNI method in stochastic settings, that may enable their applicability to GAN optimization.

\bibliography{gni}

\begin{thebibliography}{39}
\providecommand{\natexlab}[1]{#1}
\providecommand{\url}[1]{\texttt{#1}}
\expandafter\ifx\csname urlstyle\endcsname\relax
  \providecommand{\doi}[1]{doi: #1}\else
  \providecommand{\doi}{doi: \begingroup \urlstyle{rm}\Url}\fi

\bibitem[Arjovsky et~al.(2017)Arjovsky, Chintala, and Bottou]{GANobjmod2}
Arjovsky, M., Chintala, S., and Bottou, L.
\newblock {Wasserstein GAN}.
\newblock \emph{arXiv preprint}, arXiv:1701.07875, 2017.

\bibitem[Basar \& Olsder(1999)Basar and Olsder]{basar1999dynamic}
Basar, T. and Olsder, G.~J.
\newblock \emph{Dynamic noncooperative game theory}, volume~23.
\newblock Siam, 1999.

\bibitem[Bervoets et~al.(2018)Bervoets, Bravo, and Faure]{BervBravFaure18}
Bervoets, S., Bravo, M., and Faure, M.
\newblock Learning with minimal information in continuous games.
\newblock \emph{arXiv}, arXiv:1806.11506, 2018.
\newblock URL \url{https://arxiv.org/abs/1806.11506}.

\bibitem[Contreras et~al.(2004)Contreras, Klusch, and Krawczyk]{Krawczyk2004}
Contreras, J., Klusch, M., and Krawczyk, J.
\newblock Numerical solutions to nash-cournot equilibria in coupled constraint
  electricity markets.
\newblock \emph{IEEE Transactions on Power Systems}, 19:\penalty0 195–206,
  2004.

\bibitem[Daskalakis \& Panageas(2018)Daskalakis and Panageas]{DaskalakisGAN2}
Daskalakis, C. and Panageas, I.
\newblock The limit points of (optimistic) gradient descent in min-max
  optimization.
\newblock \emph{arXiv preprint}, arXiv:1807.03907v1, 2018.

\bibitem[Daskalakis et~al.(2018)Daskalakis, Ilyas, Syrgkanis, and
  Zeng]{DaskalakisGAN1}
Daskalakis, C., Ilyas, A., Syrgkanis, V., and Zeng, H.
\newblock Training gans with optimism.
\newblock \emph{arXiv preprint}, arXiv:1711.00141v2, 2018.

\bibitem[Dumoulin et~al.(2016)Dumoulin, Belghazi, Poole, Lamb, Arjovsky,
  Mastropietro, and Courville]{GANprobinf1}
Dumoulin, V., Belghazi, I., Poole, B., Lamb, A., Arjovsky, M., Mastropietro,
  O., and Courville, A.
\newblock Adversarially learned inference.
\newblock \emph{arXiv preprint}, arXiv:1606.00704, 2016.

\bibitem[Facchinei \& Kanzow(2007)Facchinei and Kanzow]{FaccKanzow07}
Facchinei, F. and Kanzow, C.
\newblock Generalized nash equilibrium problems.
\newblock \emph{4OR}, 5\penalty0 (3):\penalty0 173--210, 2007.

\bibitem[Facchinei \& Kanzow(2010)Facchinei and Kanzow]{FaccKanzow10}
Facchinei, F. and Kanzow, C.
\newblock Penalty methods for the solution of generalized nash equilibrium
  problems.
\newblock \emph{SIAM Journal on Optimization}, 20\penalty0 (5):\penalty0
  2228--2253, 2010.

\bibitem[Facchinei \& Pang(2003{\natexlab{a}})Facchinei and
  Pang]{finiteVIbook1}
Facchinei, F. and Pang, J.-S.
\newblock \emph{Finite-Dimensional Variational Inequalities and Complementarity
  Problems, Volume I}.
\newblock Springer, New York, NY, 2003{\natexlab{a}}.

\bibitem[Facchinei \& Pang(2003{\natexlab{b}})Facchinei and
  Pang]{finiteVIbook2}
Facchinei, F. and Pang, J.-S.
\newblock \emph{Finite-Dimensional Variational Inequalities and Complementarity
  Problems, Volume II}.
\newblock Springer, New York, NY, 2003{\natexlab{b}}.

\bibitem[Fedus et~al.(2018)Fedus, Rosca, Lakshminarayanan, Dai, Mohamed, and
  Goodfellow]{Fedus18}
Fedus, W., Rosca, M., Lakshminarayanan, B., Dai, A.~M., Mohamed, S., and
  Goodfellow, I.
\newblock Many paths to equilibrium: Gans do not need to decrease a divergence
  at every step.
\newblock In \emph{ICLR}, 2018.

\bibitem[Gidel et~al.(2018)Gidel, Berard, Vignoud, Vincent, and
  Lacoste-Julien]{VIGAN}
Gidel, G., Berard, H., Vignoud, G., Vincent, P., and Lacoste-Julien, S.
\newblock A variational inequality perspective on generative adversarial
  networks.
\newblock \emph{arXiv preprint}, arXiv:1802.10551v3, 2018.

\bibitem[Goodfellow(2016)]{Goodfellow16}
Goodfellow, I.
\newblock Nips 2016 tutorial: Generative adversarial networks.
\newblock \emph{arXiv preprint}, arXiv:1701.00160, 2016.

\bibitem[Goodfellow et~al.(2014)Goodfellow, Pouget-Abadie, Mirza, Xu,
  Warde-Farley, S.~Ozair, and Bengio]{GANfirst}
Goodfellow, I., Pouget-Abadie, J., Mirza, M., Xu, B., Warde-Farley, D.,
  S.~Ozair, A.~C., and Bengio, Y.
\newblock Generative adversarial nets.
\newblock In \emph{Advances in neural information processing systems}, pp.\
  2672--2680, 2014.

\bibitem[Gulrajani et~al.(2017)Gulrajani, Ahmed, Arjovsky, Dumoulin, and
  Courville]{GANobjmod3}
Gulrajani, I., Ahmed, F., Arjovsky, M., Dumoulin, V., and Courville, A.
\newblock {Improved training of Wasserstein GANs}.
\newblock \emph{arXiv preprint}, arXiv:1704.00028, 2017.

\bibitem[Isola et~al.(2016)Isola, Zhu, Zhou, and Efros]{GANimage2image}
Isola, P., Zhu, J.-Y., Zhou, T., and Efros, A.~A.
\newblock Image-to-image translation with conditional adversarial networks.
\newblock \emph{arXiv preprint}, arXiv:1611.07004, 2016.

\bibitem[Iusem et~al.(2017)Iusem, Jofr\'{e}, Oliveira, and
  Thompson]{IusemJofre17}
Iusem, A.~N., Jofr\'{e}, A., Oliveira, R.~I., and Thompson, P.
\newblock Extragradient method with variance reduction for stochastic
  variational inequalities.
\newblock \emph{SIAM Journal on Optimization}, 27\penalty0 (2):\penalty0
  686--724, 2017.

\bibitem[Karimi et~al.(2018)Karimi, Nutini, and Schmidt]{PLinequality}
Karimi, H., Nutini, J., and Schmidt, M.
\newblock Linear convergence of gradient and proximal-gradient methods under
  the polyak-lojasiewicz condition.
\newblock \emph{arXiv preprint}, arXiv:1608.04636v3, 2018.

\bibitem[Kingma \& Ba(2014)Kingma and Ba]{kingma2014adam}
Kingma, D.~P. and Ba, J.
\newblock Adam: A method for stochastic optimization.
\newblock \emph{arXiv preprint arXiv:1412.6980}, 2014.

\bibitem[Korpelevich(1976)]{korpelevich1976extragradient}
Korpelevich, G.
\newblock The extragradient method for finding saddle points and other
  problems.
\newblock \emph{Matecon}, 12:\penalty0 747--756, 1976.

\bibitem[Li et~al.(2017)Li, Schwing, K.-C.!Wang, and Zemel]{DualGAN}
Li, Y., Schwing, A., K.-C.!Wang, and Zemel, R.
\newblock Dualizing gans.
\newblock In \emph{NIPS}, 2017.

\bibitem[{\L}ojasiewicz(1963)]{Lojasiewicz63}
{\L}ojasiewicz, S.
\newblock A topological property of real analytic subsets (in french).
\newblock \emph{Coll. du CNRS, Les \'{e}quations aux d\'{e}riv\'{e}es
  partielles}, pp.\  87–89, 1963.

\bibitem[Luo \& Tseng(1993)Luo and Tseng]{LuoTseng93}
Luo, Z.-Q. and Tseng, P.
\newblock Error bounds and convergence analysis of feasible descent methods: A
  general approach.
\newblock \emph{Ann. Oper. Res.}, pp.\  157–178, 1993.

\bibitem[Mertikopoulos \& Zhou(2019)Mertikopoulos and Zhou]{MertZhou19}
Mertikopoulos, P. and Zhou, Z.
\newblock Learning in games with continuous action sets and unknown payoff
  functions.
\newblock \emph{Mathematical Programming}, 173\penalty0 (1--2):\penalty0
  465--507, 2019.

\bibitem[Mertikopoulos et~al.(2019)Mertikopoulos, Zenati, Lecouat, Foo,
  Chandrasekhar, and Piliouras]{MertZenLec19}
Mertikopoulos, P., Zenati, H., Lecouat, B., Foo, C.-S., Chandrasekhar, V., and
  Piliouras, G.
\newblock Optimistic mirror descent in saddle-point problems: Going the extra
  (gradient) mile.
\newblock In \emph{International Conference on Learning Representation (ICLR)},
  2019.

\bibitem[Mescheder et~al.(2017)Mescheder, Nowozin, and Geiger]{GANprobinf2}
Mescheder, L., Nowozin, S., and Geiger, A.
\newblock Adversarial variational bayes: Unifying variational autoencoders and
  generative adversarial networks.
\newblock \emph{arXiv preprint}, arXiv:1701.04722, 2017.

\bibitem[Nagarajan \& Kolter(2017)Nagarajan and Kolter]{NagarajanKolter}
Nagarajan, V. and Kolter, J.
\newblock Gradient descent gan optimization is locally stable.
\newblock \emph{arXiv preprint}, arXiv:1706.04156v3, 2017.

\bibitem[Nikaido \& Isoda(1955)Nikaido and Isoda]{NikaidoIsoda}
Nikaido, H. and Isoda, K.
\newblock Note on noncooperative convex games.
\newblock \emph{Pacific Journal of Mathematics}, 5\penalty0 (1):\penalty0
  807–815, 1955.

\bibitem[Nocedal \& Wright(2006)Nocedal and Wright]{NocWrightbook}
Nocedal, J. and Wright, S.~J.
\newblock \emph{Numerical Optimization}.
\newblock Springer Series in Operations Research and Financial Engineering.
  Springer, 2nd edition, 2006.

\bibitem[Polyak(1963)]{Polyak63}
Polyak, B.~T.
\newblock Gradient methods for minimizing functionals (in russian).
\newblock \emph{Zh. Vychisl. Mat. Mat. Fiz.}, pp.\  643–653, 1963.

\bibitem[Radford et~al.(2015)Radford, Metz, and Chintala]{GANarchmod}
Radford, A., Metz, L., and Chintala, S.
\newblock Unsupervised representation learning with deep convolutional
  generative adversarial networks.
\newblock \emph{arXiv preprint}, 2015.

\bibitem[Rakhlin \& Sridharan(2013)Rakhlin and Sridharan]{RakhlinSridharan}
Rakhlin, A. and Sridharan, K.
\newblock Online learning with predictable sequences.
\newblock In \emph{COLT 2013 - The 26th Annual Conference on Learning Theory},
  pp.\  993--1019, 2013.

\bibitem[Salimans et~al.(2016)Salimans, Goodfellow, Zaremba, Cheung, Radford,
  and Chen]{GANobjmod1}
Salimans, T., Goodfellow, I.~J., Zaremba, W., Cheung, V., Radford, A., and
  Chen, X.
\newblock {Improved techniques for training GANs}.
\newblock \emph{CoRR}, abs/1606.03498, 2016.

\bibitem[Tzeng et~al.(2017)Tzeng, Hoffman, Saenko, and Darrell]{GANdomain}
Tzeng, E., Hoffman, J., Saenko, K., and Darrell, T.
\newblock Adversarial discriminative domain adaptation.
\newblock \emph{arXiv preprint}, arXiv:1702.05464, 2017.

\bibitem[Uryasev \& Rubinstein(1994)Uryasev and
  Rubinstein]{NI_UryasevRubinstein}
Uryasev, S. and Rubinstein, R.
\newblock On relaxation algorithms in computation of noncooperative equilibria.
\newblock \emph{IEEE Transactions on Automatic Control}, 39\penalty0
  (6):\penalty0 1263–1267, 1994.

\bibitem[von Heusinger \& Kanzow(2009{\natexlab{a}})von Heusinger and
  Kanzow]{HeusingerKanzow1}
von Heusinger, A. and Kanzow, C.
\newblock Optimization reformulations of the generalized nash equilibrium
  problem using nikaido-isoda-type functions.
\newblock \emph{Comput. Optim. Appl.}, 43\penalty0 (3):\penalty0 353–377,
  2009{\natexlab{a}}.

\bibitem[von Heusinger \& Kanzow(2009{\natexlab{b}})von Heusinger and
  Kanzow]{HeusingerKanzow2}
von Heusinger, A. and Kanzow, C.
\newblock Relaxation methods for generalized nash equilibrium problems with
  inexact line search.
\newblock \emph{Journal of Optimization Theory and Applications}, 143\penalty0
  (1):\penalty0 159–183, 2009{\natexlab{b}}.

\bibitem[Yadav et~al.(2017)Yadav, Shah, Xu, Jacobs, and Goldstein]{GANstablize}
Yadav, A., Shah, S., Xu, Z., Jacobs, D., and Goldstein, T.
\newblock Stabilizing adversarial nets with prediction methods.
\newblock \emph{arXiv preprint}, arXiv:1705.07364, 2017.

\end{thebibliography}
\bibliographystyle{icml2019}

 \appendix
 
\section{Residual Minimization}
Lemma 1 (in the main paper) also suggests another possible function for minimization, namely 
$\Phi(\bfx) = \frac{1}{2} \sum\limits_{i=1}^2 \|\nabla_i f_i(\bfx)\|^2$.  We can state a 
result that is analogous to Theorem 1.

\begin{theorem}\label{thm:zeroVEqForderCrit1}
The global minimizers of $\Phi(\bfx)$ are all first-order NE points, i.e., 
$\{x^\star \,|\, \Phi(\bfx^\star) = 0\} = {\cal S}^{SNP}$.  If the individual functions $f_i$ are convex then the global minimizers 
of $\Phi(\bfx)$ are precisely the set ${\cal S}^{NE}$. 
\end{theorem}
Denote by $F(\bfx) = \begin{bmatrix} \nabla_1 f_1(\bfx) \\ \nabla_2 f_2(\bfx) \end{bmatrix}$ 
the vector function of the first-order stationary conditions for each of the players.  So 
$\Phi(\bfx) = \frac{1}{2} \|F(\bfx)\|^2$.  
The gradient of $\Phi(\bfx)$ is given by 
\begin{equation}
\begin{aligned}
\nabla \Phi(\bfx) &=  \nabla F(\bfx) F(\bfx) \\
&=  \begin{bmatrix} 
	\nabla^2_{11} f_1(\bfx)  & \nabla^2_{12} f_2(\bfx) \\
	\nabla^2_{21} f_1(\bfx)  & \nabla^2_{22} f_2(\bfx) 
\end{bmatrix} \begin{bmatrix}
	\nabla_1 f_1(\bfx) \\
	\nabla_2 f_2(\bfx) 
\end{bmatrix}.
\end{aligned}
\label{defGradPhi}\end{equation}
The Hessian of the function $\Phi(\bfx)$ is 
\begin{equation}
	\nabla^2 \Phi(\bfx) =  \left( \sum\limits_{j=1}^n F_j(\bfx) \nabla^2 F_j(\bfx)  + 
	\nabla F(\bfx) \nabla F(\bfx)^T \right).
	\label{defHessPhi}
\end{equation}

Consider the gradient descent iteration for minimizing $\Phi(\bfx)$ with stepsize 
$\rho > 0$
\begin{equation}
	x^{k+1} = x^k - \rho \nabla \Phi(\bfx^k).	\label{gradDescentPhi}
\end{equation}

We can state the following convergence result for the gradient descent iterations.

\begin{theorem}
Suppose $\nabla \Phi(\bfx)$ is $ L_{\Phi}$-Lipschitz continuous.  
Let $\rho = \frac{1}{ L_{\Phi}}$. Then, the $\{x^k\}$ generated 
by~\eqref{gradDescentPhi} converges sublinearly to $x^\star$ a first-order critical point of 
$\Phi(\bfx)$, $\nabla \Phi(\bfx^\star) = 0$.  If $\Phi(\bfx) \leq \frac{1}{2\mu} 
\|\nabla \Phi(\bfx)\|^2$ then the sequence $\{x^k\}$ converges linearly to a 
$x^\star \in {\cal S}^{SNP}$.
\end{theorem}

\begin{proof}
From Lipschitz continuity of $\nabla \Phi(\bfx)$ 
\begin{equation}
\begin{aligned}
\Phi(\bfx^{k+1}) \leq& \Phi(\bfx^k) + \nabla \Phi(\bfx^k)^T 
(\bfx^{k+1}-x^k)  \\
&+ \frac{ L_{\Phi}}{2} \|x^{k+1}-x^k\|^2 \\
\leq& \Phi(\bfx^k,) - \rho (1  - \frac{ \rho L_{\Phi} }{2} ) \|\nabla \Phi(\bfx)\|^2 \\
\leq& \Phi(\bfx^k) - \frac{1}{2 L_{\Phi}}  \|\nabla \Phi(\bfx)\|^2. 
\end{aligned}\label{descPhiIter}
\end{equation}
Telescoping the sum and $k = 0,...,K$ obtain
\begin{equation}
 \Phi(\bfx^{K+1}) \leq \Phi(\bfx^0) -  \frac{1}{2 L_{\Phi}} \sum\limits_{k=0}^{K} 
 \|\Phi(\bfx^K)\|^2.
 \end{equation}
 Since $\Phi(\bfx)$ is bounded below by $0$ we have that 
 \[\begin{aligned}
 &\frac{1}{2 L_{\Phi}} \sum\limits_{k=0}^{K} 
 \| \nabla \Phi(\bfx^K)\|^2 \leq  \Phi(\bfx^0) - \Phi(\bfx^{K+1}) \leq \Phi(\bfx^0) \\
 &\implies \frac{1}{2 L_{\Phi}} \min\limits_{k \in \{0,\ldots,K\}} \| \nabla \Phi(\bfx^k)\|^2 
 \leq \frac{\Phi(\bfx^0)}{K+1}.
\end{aligned} \]
This proves the claim on sublinear convergence to a first-order stationary point of $\Phi(\bfx)$.  
Suppose $\Phi(\bfx) \leq \frac{1}{2\mu} \|\nabla \Phi(\bfx)\|^2$ holds.  Substituting 
in~\eqref{descPhiIter} obtain
\begin{equation}
\begin{aligned}
\Phi(\bfx^{k+1}) \leq& \left( 1 -  \frac{\mu}{L_{\Phi}} \right)\Phi(\bfx)
\end{aligned}\label{descPhiIter1}
\end{equation}
which proves the claim on linear convergence.
\end{proof}

In the following we provide specific conditions under which the bound 
$\Phi(\bfx) \leq \frac{1}{2\mu} \|\nabla \Phi(\bfx)\|^2$ holds. 
\begin{itemize}
	\item Suppose the function $f_i$ are quadratic then the discussion following Theorem 3 applies.
	\item	Suppose the function $F(x)$ is strongly monotone,  
	$(F(x) - F(\hat{x}))^T(x-\hat{x}) \geq \beta \|x - \hat{x}\|^2$.  This implies that the $f_i(x)$ 
	are $\beta$-strongly convex.  Then, it follows that $\nabla F(x) \succeq \beta I_n$ for 
	all $x \in \R^n$.  This also provides the following bound 
	\begin{equation}
		\|\nabla \Phi(x;\eta)\|^2 \geq ( \beta)^2 \|F(x)\|^2 = 2\beta^2 \Phi(x).
	\end{equation}
	Hence, $\mu = \beta^2$.
\end{itemize}


\section{Additional Experiments}
In this Section, we present more empirical results for the four different games that were discussed in the main paper which help understand the convergence behavior of the proposed method. More concretely, the results validate the results for convergence rate and the quality of solutions for the different games discussed in the main paper.
\subsection{Convergence Rate for Bilinear and Quadratic Games:}
We provide plots that suggest linear convergence rate for bilinear and strongly-convex quadratic games as was described in the main paper. For both cases we use $20$-d variables for both players which are initialized arbitrarily. From the plots shown in Figure~\ref{fig:convergence_rate}, we observe that $V$ function decays linearly to close to zero and then it slows down as the gradient of $V$ starts to vanish (suggested by Theorem $2$ in the main paper). It is noted that the guarantees for linear convergence are for the $V$ function (and not for $\nabla f$) and thus we skip plots for $\nabla f$.
\begin{figure}[!h]
    \centering
    \includegraphics[width=4cm]{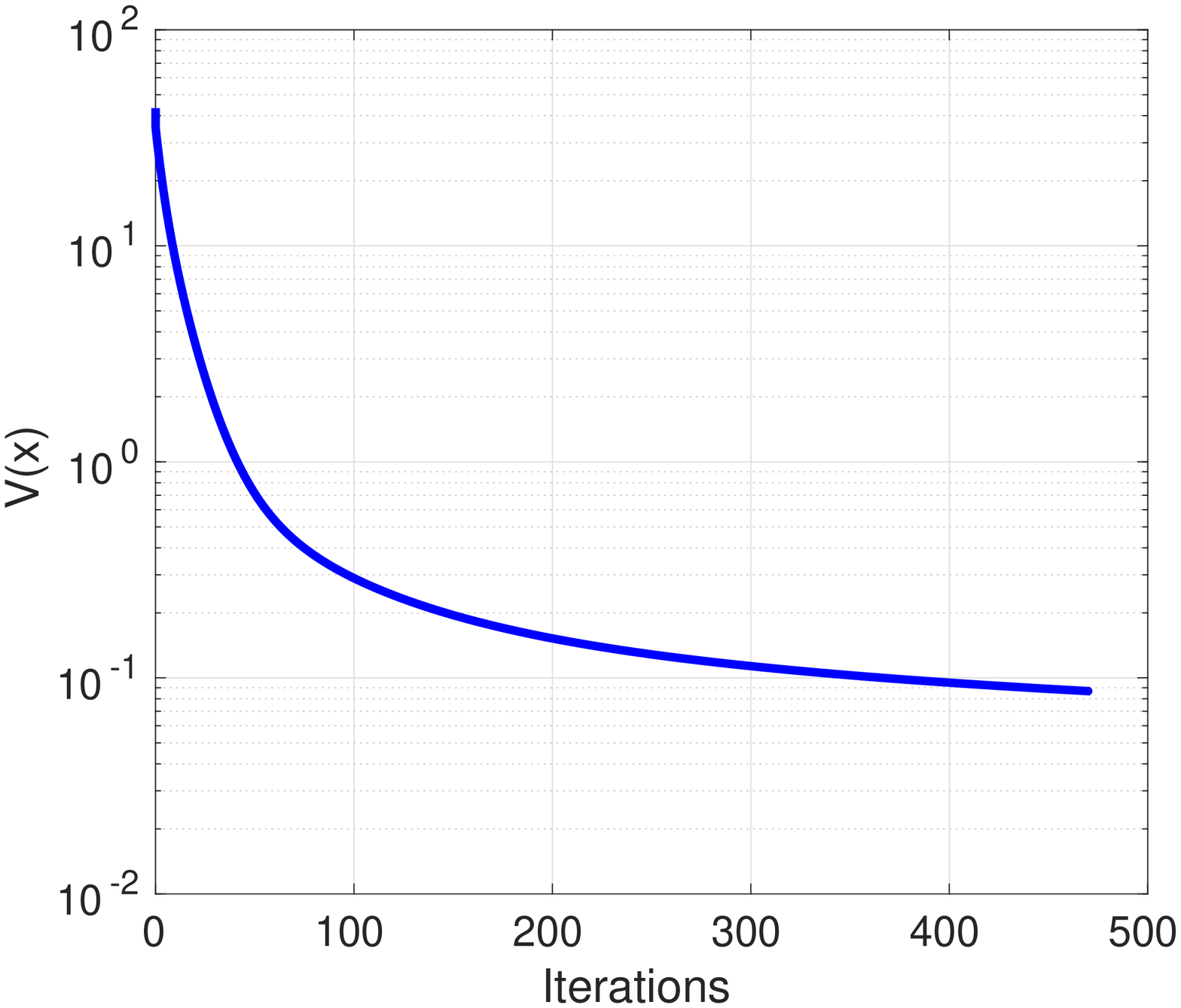}
    \includegraphics[width=4cm]{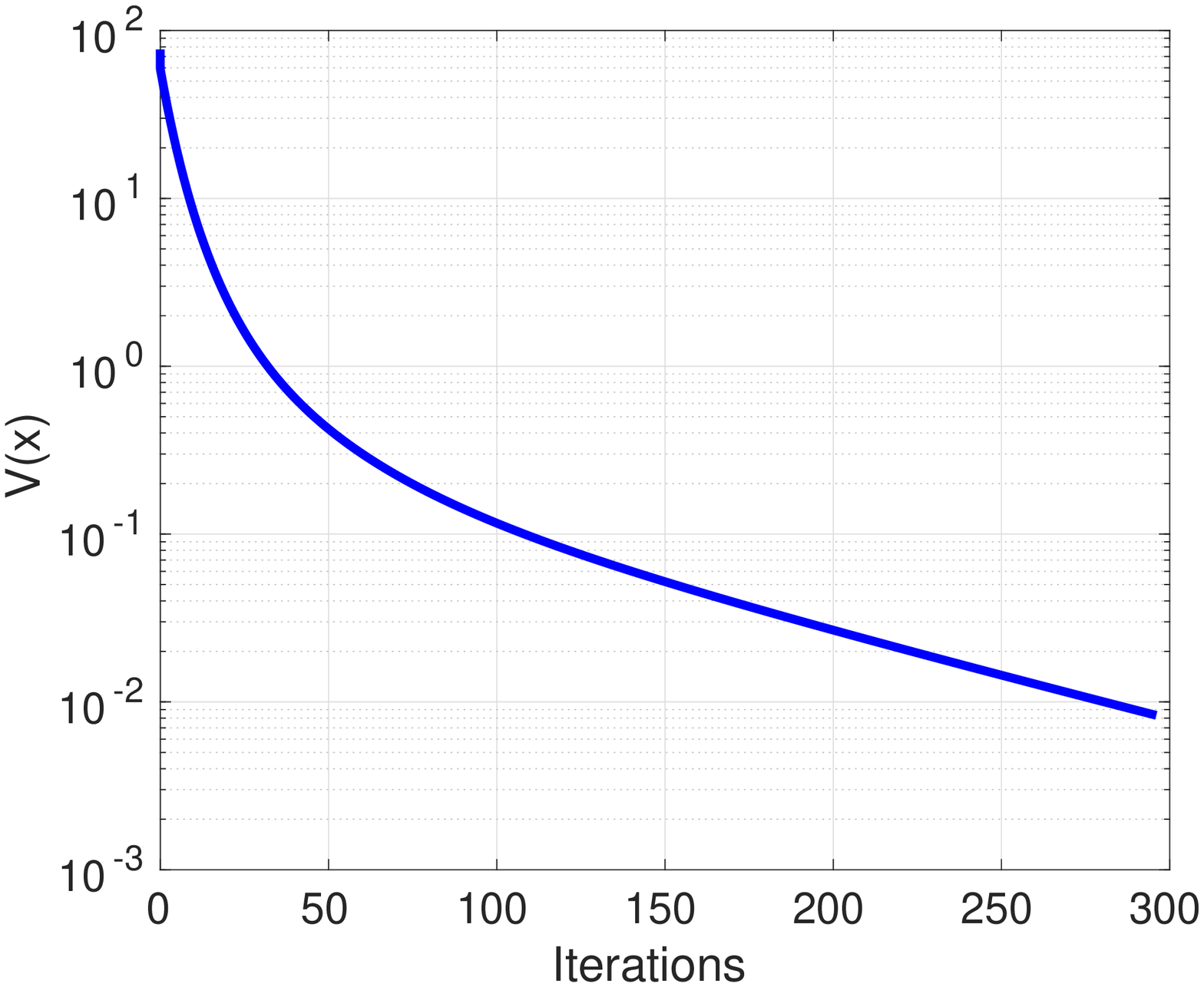}
    \caption{Convergence rate for bilinear and convex quadratic games using the GNI method. Left: Decay of $V$ function for Bilinear Game. Right: Decay of $V$ function for Strongly-convex Quadratic Game.}
    \label{fig:convergence_rate}
\end{figure}
\subsection{Two-Player Quadratic Games:}
We describe an experiment for non-convex two-player quadratic games with indefinite $Q$ matrices for both players. We show the decay of the gradient and the $V$ function for the GNI formulation. The other optimization algorithms are seen to be diverging for the indefinite cases (as was shown in the main text) and thus are not shown here. We used $50-$d data, the same stepsizes $\eta=\max_i(\norm{Q_i})$ and $\rho=0.01$ for GNI. The methods are initialized randomly from $N(0,I)$. For clarification, we show the plot on log scale. As can be observed from the plots in Figure~\ref{fig:quad_indefinite_convergence}, $\nabla f$ goes to zero as $V$ goes to zero.

\begin{figure}[!h]
    \centering
    \includegraphics[width=4cm]{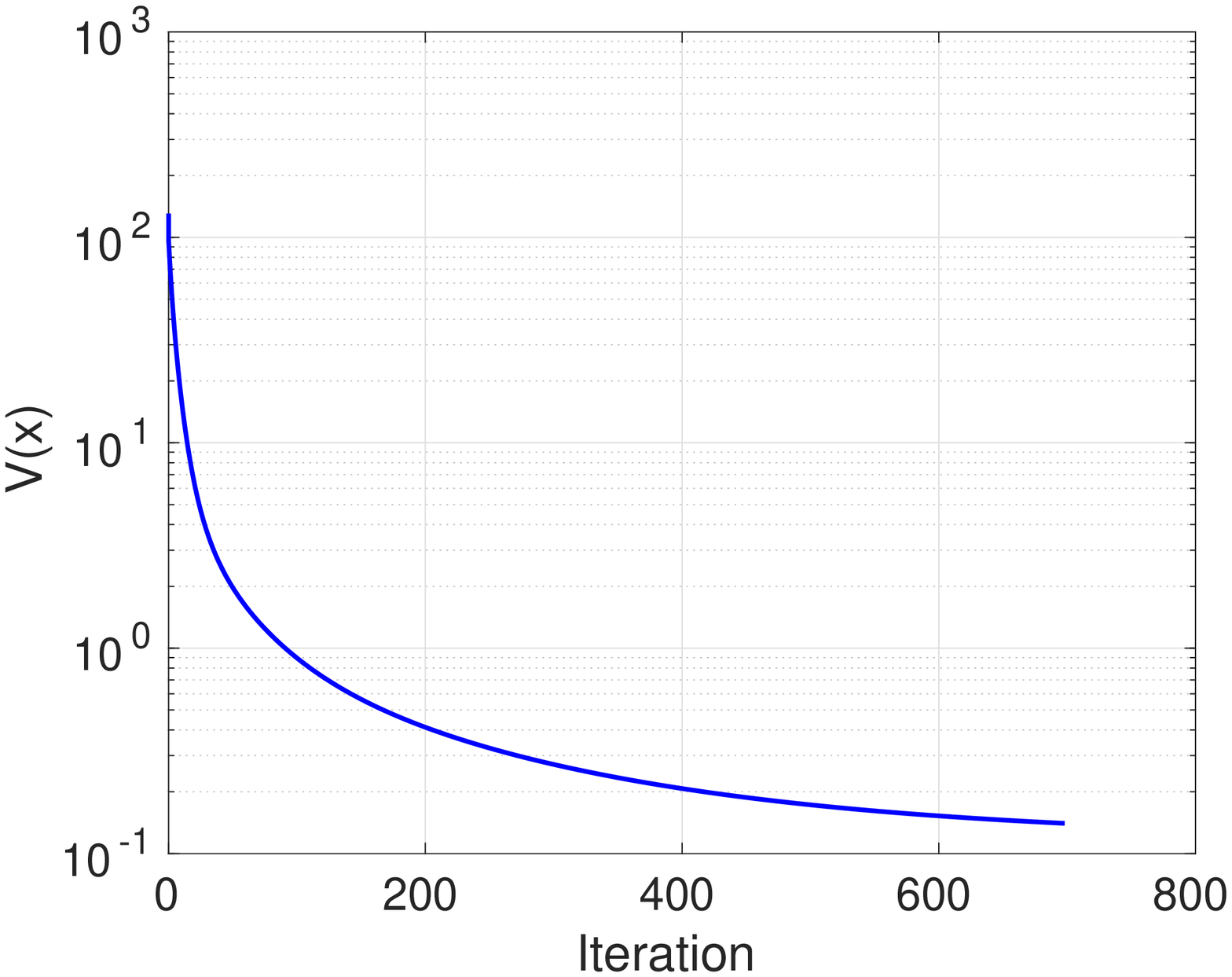}
    \includegraphics[width=4cm]{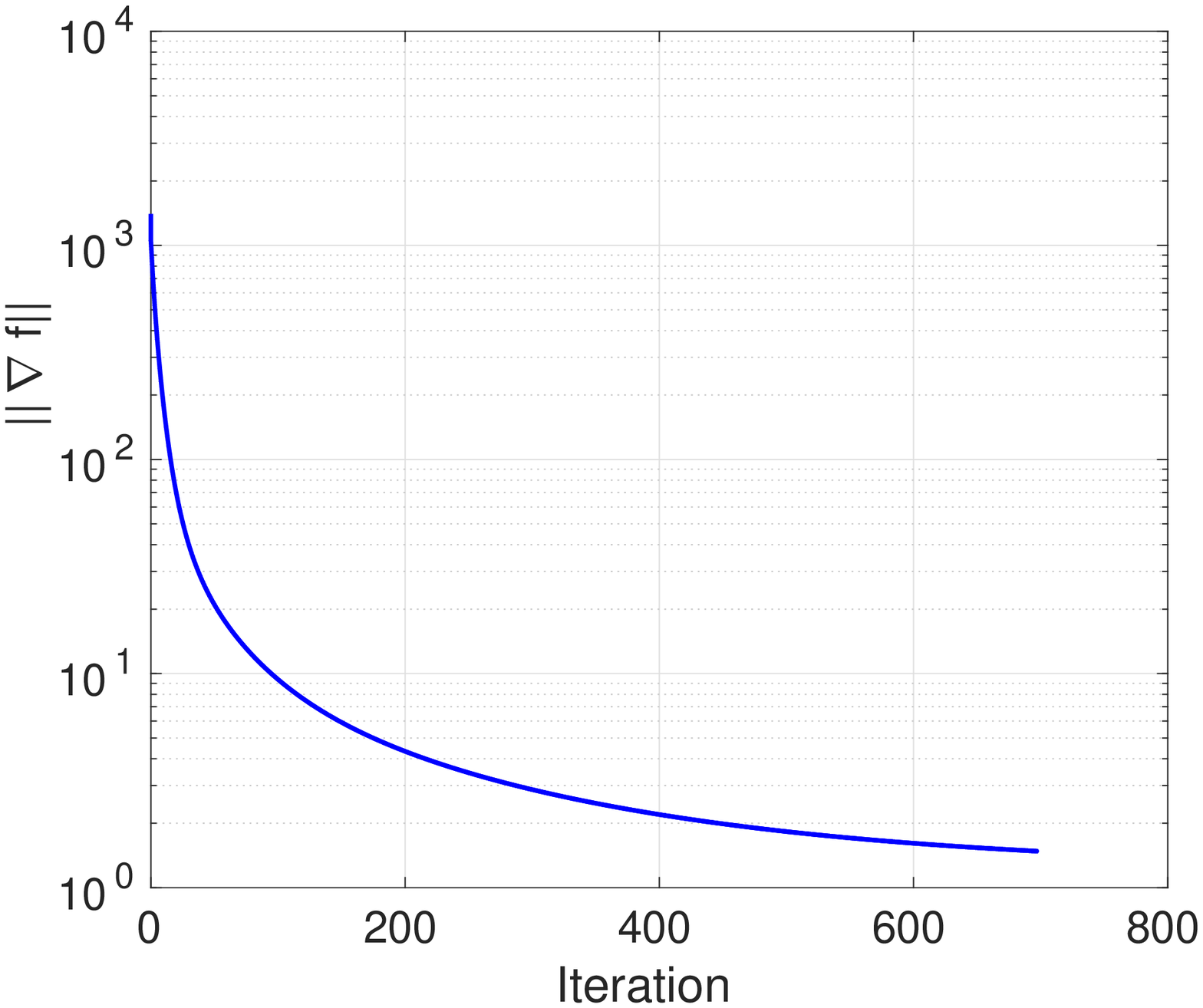}
    \caption{Convergence of GNI method for non-convex quadratic game setting shown on a semi-log plot for clarity. Left: Decay of $V$ function. Right: Decay of the $\nabla f$.}
    \label{fig:quad_indefinite_convergence}
\end{figure}

\begin{figure}[!h]
    \centering
    \includegraphics[width=4cm]{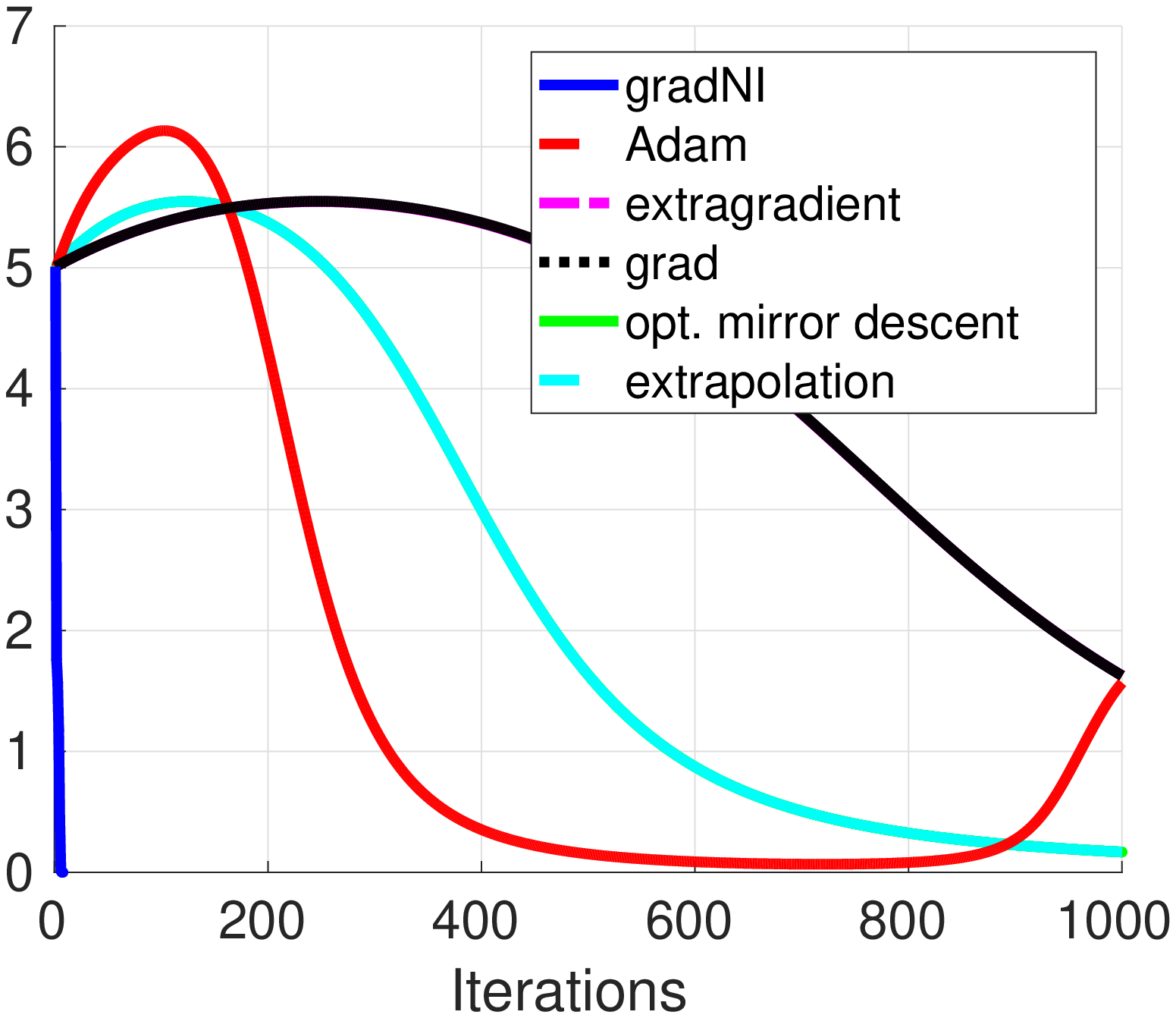}
    \includegraphics[width=4cm]{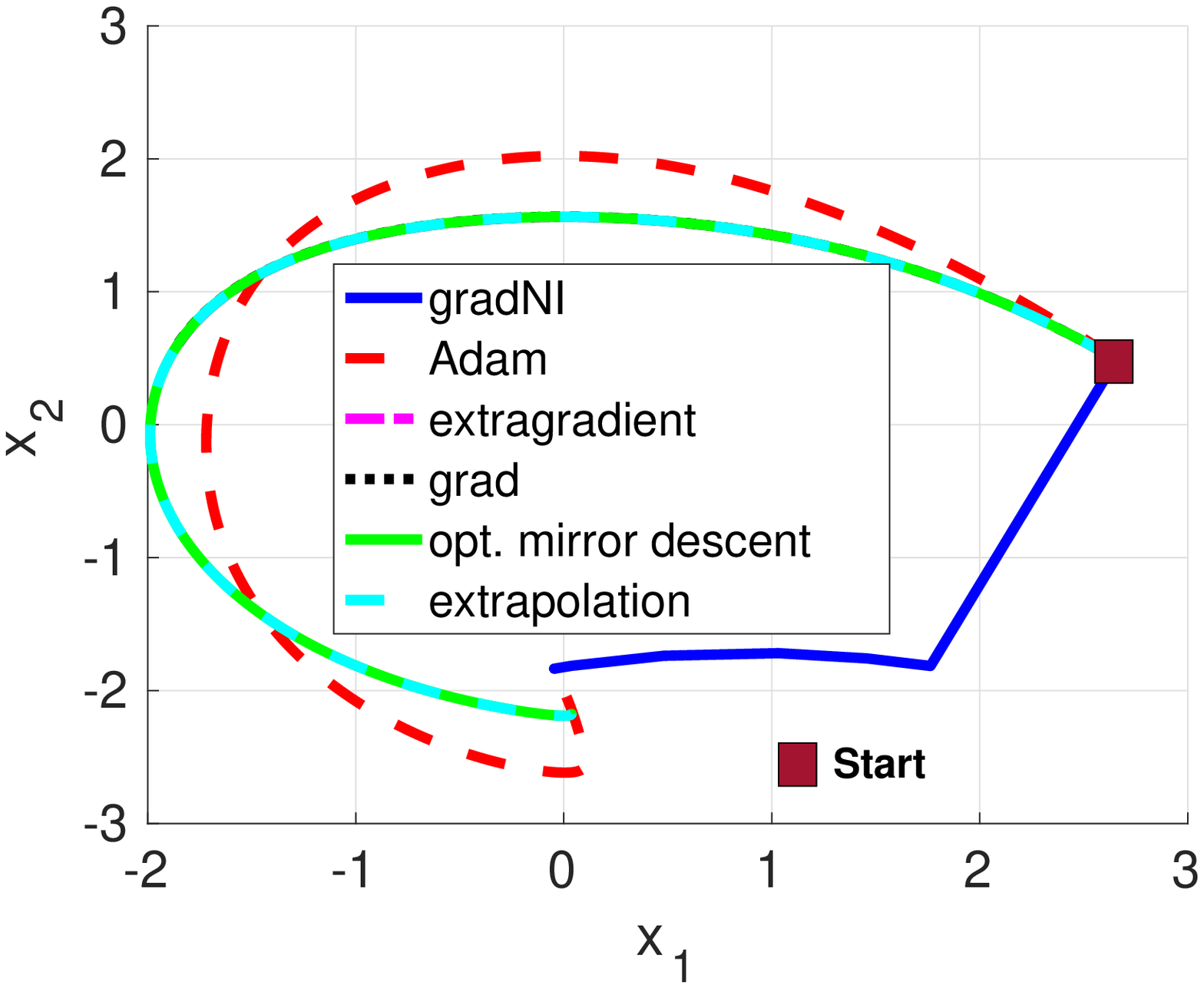}
    \caption{Convergence of GNI against other methods on the Dirac-Delta GAN. Left: Convergence of different methods seen by the decay of $\nabla f$. Right: Trajectory of the two players to the optima. }
    \label{fig:ddgan_convergence_with}
\end{figure}

\begin{figure*}[ht]
\centering
 \begin{tabular}{||c c c c c c c||} 
 \hline
 Algorithm & GNI & Adam & ExGrad & Grad & OMD &ExPol \\ [0.5ex] 
 \hline\hline
 Mean Error & $2.11$ & $0.64$ & $2.17$ & $2.18$ & $1.87$ & $1.87$\\ 
 \hline
 Mean number of Iterations & $77$ & $3048$ & $10000$ & $10000$ & $10000$ & $10000$  \\
 \hline
 \end{tabular}
\caption{\label{tab:error_stats}Error Statistics for GNI compared against other techniques for the Dirac Delta GAN}
\end{figure*}
\subsection{Linear GAN:} We also show some additional results for the Linear GAN which suggests convergence of the proposed method to a NE. The second derivative of the objective function for both the players is positive semidefinite (see Equation $(23)$ in the main paper)  indicating all stationary points are minimas. In the following plots in Figure~\ref{fig:linear_gan_convergence}, we show the convergence of the $V$ function and the $||\nabla f||$ for the GNI formulation. We observe very fast convergence for both the $V$ and the $||\nabla f||$ indicating convergence to a SNP. Additionally, since all SNPs are NEs in this particular setting, the GNI converges to a NE.

\begin{figure}[h]
    \centering
    \includegraphics[width=4cm]{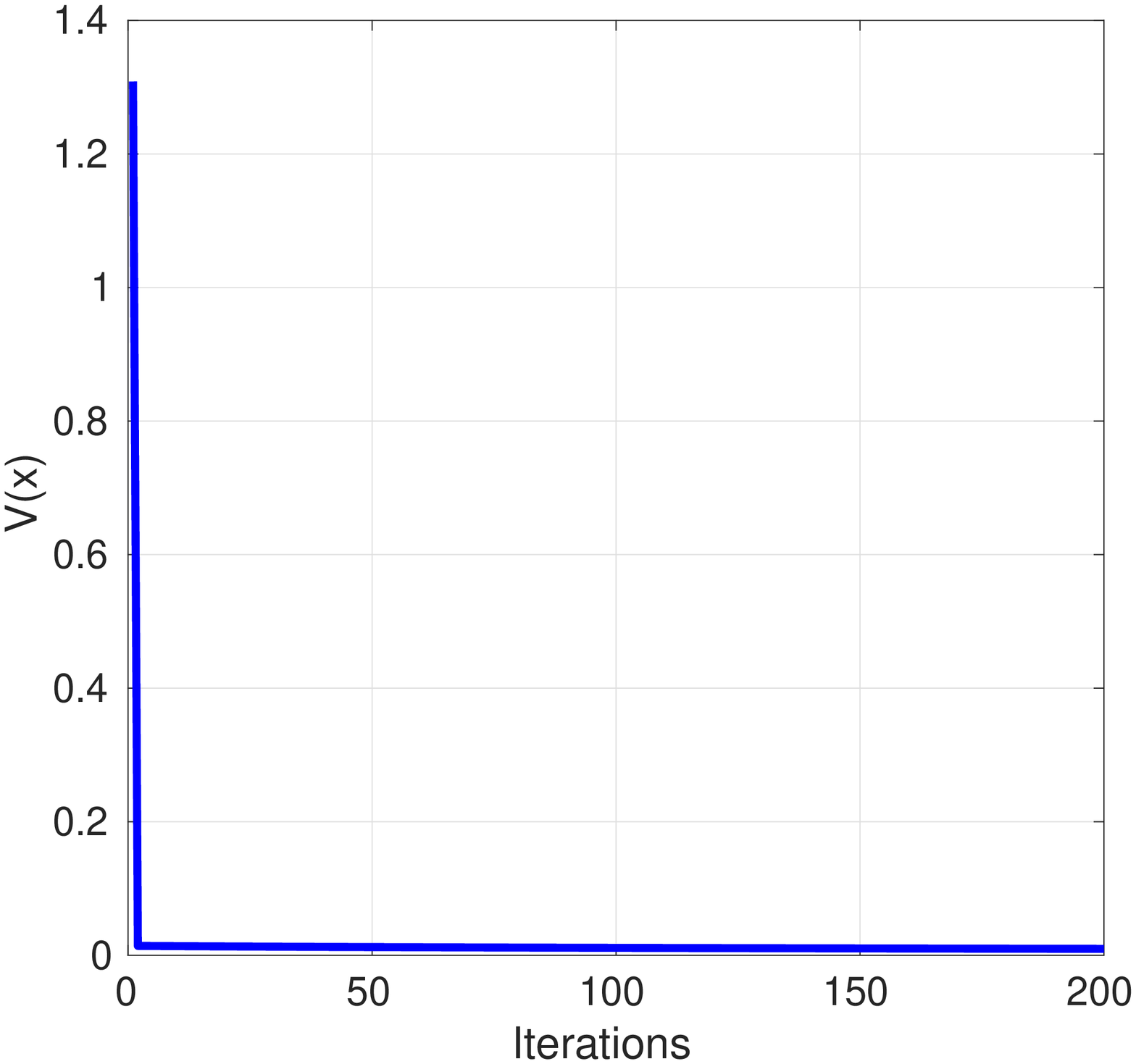}
    \includegraphics[width=4cm]{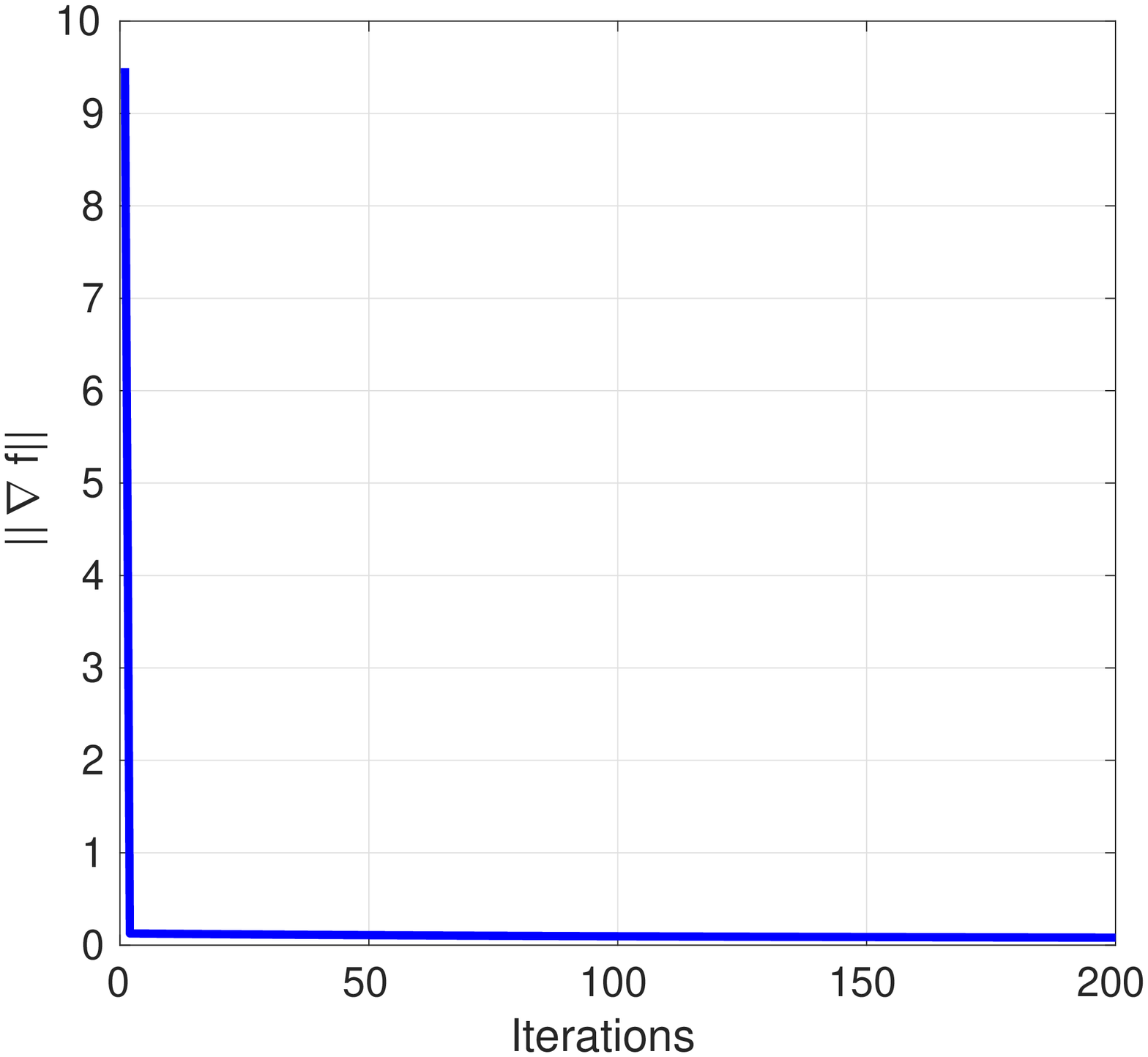}
    \caption{Convergence of $V$ function and $||\nabla f||$ for the Linear GAN discussed in the main paper. Left: Decay of $V$ function. Right: Decay of the $\nabla f$.}
    \label{fig:linear_gan_convergence}
\end{figure}
\subsection{Dirac Delta GAN:}
In this section, we show another experiment for the Dirac Delta GAN that was discussed in the main text. All the parameters for all optimizers are kept constant as in the main text for Dirac Delta GAN. In Figure~\ref{fig:ddgan_convergence_with}, we see the convergence of $\nabla f$ as well as the trajectories followed by the two players to the NE. All the methods converge to the same optima- however, the GNI converges faster than any other method. As observed in the convex quadratic case, we see all descent methods following the same trajectory except for the GNI and Adam. However, it was observed that the GNI and the other algorithms do not converge to the same solution when initialized arbitrarily. To investigate this, we perform an experiment where the game was initialized randomly from $1000$ initial conditions in a square region in $[-4,4]\times [-4,4]$. The error from the ground truth was computed after $10000$ iterations or up on convergence (the minimum of two). Results of the experiment are shown as a table in Figure~\ref{tab:error_stats}. It is observed that the game doesn't converge to the known ground truth for the game-- Adam is able to get closest to the ground truth while GNI converges to a stationary Nash point much faster than all other algorithms. This behavior could be explained by recalling that GNI is using $V$ function to descend and thus, converges to the closest stationary Nash point where $V$ vanishes.

\end{document}